\documentclass{article}

\usepackage{microtype}
\usepackage{graphicx}
\usepackage{subfig}
\usepackage{booktabs} 

\usepackage{arxiv}
\usepackage[margin=1.48in]{geometry}
\usepackage[parfill]{parskip}

\usepackage[inline]{enumitem}
\usepackage{amsfonts} 
\usepackage{amsmath}
\usepackage{amssymb}
\usepackage{algorithm}
\usepackage{algorithmicx}
\usepackage{algpseudocode}
\usepackage{float}
\usepackage{wrapfig}
\usepackage{natbib} 
\usepackage{mathtools}
\usepackage{amsthm}
\usepackage{color}
\usepackage{bbm}
\allowdisplaybreaks
\usepackage{xcolor}

\usepackage{hyperref}

\newcommand{\LL}{\left}
\newcommand{\RR}{\right}
\newcommand{\E}{\mathbb E}
\newcommand{\Z}{\mathbb Z}
\newcommand{\R}{\mathbb R}
\newcommand{\mA}{\mathcal A}
\newcommand{\mD}{\mathcal D}

\newcommand{\mH}{\mathcal H}
\newcommand{\mB}{\mathcal T}
\newcommand{\mS}{\mathcal S}
\newcommand{\mO}{\mathcal O}
\newcommand{\mJ}{\mathcal J}
\newcommand{\mM}{\mathcal M}
\newcommand{\mX}{\mathcal X}
\newcommand{\ma}{\textbf{a}}

\newcommand{\ms}{\textbf{s}}
\newcommand{\fJ}{\mathfrak J}
\DeclareMathOperator*{\argmax}{argmax}
\DeclareMathOperator*{\argmin}{argmin}

\newcommand{\var}{\mathrm{Var}}
\newcommand{\cov}{\mathrm{Cov}}

\newcommand{\TV}{D_{\mathrm{TV}}}
\newcommand{\CQL}{D_{\mathrm{CQL}}}
\newcommand{\loce}{\ell_1\text{-CE}}
\newcommand{\chis}[2]{\chi^2(#1\,\|#2)}
\newcommand{\COne}{{\sqrt{\var_{\pi^f}(Q(\ms,\cdot))|\mD(\ms)|}}}
\newcommand{\CTwo}{{\sqrt{\var_{d_f}\hat Q^k}}}
\newcommand{\CThr}{{\sqrt{\var_{d_f}(\mat F \omega^k)}}}
\newcommand{\dsq}{\partial_\psi Q|_{\ms,\ma}}

\newcommand{\mat}[1]{\mathbf{#1}}
\newcommand{\gQ}{\nabla_\psi \hat Q^k}
\newcommand{\lstd}{\textnormal{LSTD}}
\newcommand{\minimize}{\textnormal{minimize}}
\newcommand{\rob}{\textnormal{rob}}

\theoremstyle{definition}
\newtheorem{theorem}{Theorem}[section]

\newtheorem{lemma}[theorem]{Lemma}

\newtheorem{corollary}[theorem]{Corollary}

\newtheorem{assumption}[theorem]{Assumption}

\newtheorem{remark}[theorem]{Remark}

\providecommand{\keywords}[1]
{
  \small	
  \textbf{\textit{Keywords---}} #1
}
\title{DROMO: \\Distributionally Robust Offline Model-based Policy Optimization}
\date{}

\author{
Ruizhen Liu\\
Affiliated High School of South China Normal University\\
\texttt{liurz.jason2019@gdhfi.com}\\
\AND
Zhicong Chen \\
Affiliated High School of South China Normal University\\
\texttt{chenzc.marin2018@gdhfi.com}\\
\AND
Dazhi Zhong \\
Affiliated High School of South China Normal University\\
\texttt{zhongdz.dazhi2018@gdhfi.com}
}

\begin{document}

\maketitle

\begin{abstract}
	We consider the problem of offline reinforcement learning with model-based control, whose goal is to learn a dynamics model from the experience replay and obtain a pessimism-oriented agent under the learned model. Current model-based constraint includes explicit uncertainty penalty and implicit conservative regularization that pushes Q-values of out-of-distribution state-action pairs down and the in-distribution up. While the uncertainty estimation, on which the former relies on, can be loosely calibrated for complex dynamics, the latter performs slightly better. To extend the basic idea of regularization without uncertainty quantification, we propose  \textit{\textbf{d}istributionally \textbf{r}obust \textbf{o}ffline \textbf{m}odel-based policy \textbf{o}ptimization} (DROMO), which leverages the ideas in distributionally robust optimization to penalize a broader range of out-of-distribution state-action pairs beyond the standard empirical out-of-distribution Q-value minimization. We theoretically show that our method optimizes a lower bound on the ground-truth policy evaluation, and it can be incorporated into any existing policy gradient algorithms. We also analyze the theoretical properties of DROMO's linear and non-linear instantiations.

	\newline
\end{abstract}

\newpage
\tableofcontents

\newpage
\section{Introduction}\label{sec:intro}

The field of reinforcement learning (RL) (\cite{sutton1999reinforcement}) focuses on looking for the best sequential planning under a given environment. Specifically, under a Markov Decision Process (MDP) (\cite{PUTERMAN1990331}), the planner, commonly known as the agent, performs some actions under given states of the environment, and, by doing so, triggers the environment to spill out new states and some rewards. The goal of the agent is to accumulate as many rewards as possible. Unlike the other fields of machine learning, in which models can be trained under large, realistic, and balanced datasets and generalize well, reinforcement learning algorithms fail to enjoy the benefit of balanced datasets, and instead require costly training-time online trial-and-error to achieve good generalization results on real-world applications, which range from autopilot (\cite{shalev-shwartzSafeMultiAgentReinforcement2016}, \cite{sun2020scalability}, \cite{yuBDD100KDiverseDriving2020}), to recommendation systems (\cite{li2010contextual}), and to precision medicine (\cite{chakraborty2014dynamic}, \cite{gottesman2019guidelines}). One way to reduce the cost is to the policy under an offline (batch) regime.

Although online RL is well-understood (\cite{lattimore2020bandit}, \cite{agarwal2019reinforcement}), offline RL remains less so. Vanilla function approximation can be problematic when applied to MDPs with large and continuous state and action spaces. Granted, off-policy RL algorithms, such as DDPG (\cite{lillicrapContinuousControlDeep2019}), TRPO (\cite{schulmanTrustRegionPolicy2015}), and SAC (\cite{haarnojaSoftActorCriticOffPolicy2018}), can utilize an experience replay. However, they perform poorly without online data collection, and, even with online data collection, the function approximation is nonetheless sensitive to covariate distribution shift (\cite{vanhasseltDeepReinforcementLearning2018} \cite{fu2019diagnosing}). Furthermore, \cite{fujimotoOffPolicyDeepReinforcement2019} and \cite{fu2019diagnosing} empirically shows that DDPG fails both in theory and in practice. The most likely reason behind this is that the dataset's insufficient coverage of the environment induces covariate distributional shift, which, in turn, causes erroneous function approximation for states (value function) or state-action pairs (Q-function) (\cite{sutton1995virtues}, \cite{vanhasseltDeepReinforcementLearning2018}). Recent attempts of tackling the issue can be partitioned into two possibly overlapping categories: (i) bootstrap aggregation via an ensemble of target Q-networks to stabilize the action value approximations (\cite{agarwalOptimisticPerspectiveOffline2020}), (ii) regularizing the Q-function via policy constraint (\cite{fujimotoOffPolicyDeepReinforcement2019}, \cite{kumarStabilizingOffPolicyQLearning2019}, \cite{wuBehaviorRegularizedOffline2019}, \cite{siegelKeepDoingWhat2020}), and (iii) injecting pessimism via epistemic uncertainty quantification (\cite{luoLearningSelfCorrectablePolicies2019}, \cite{yuMOPOModelbasedOffline2020}, \cite{kumarConservativeQLearningOffline2020}, \cite{yuCOMBOConservativeOffline2021}). (ii) relies only on the dataset and restricts the learned policy not to visit states and perform actions scarcely covered by the dataset, which often induces an overly conservative actor. On the other hand, (iii) learns a dynamics model and subtracts a penalty off the Q-function or value function for states and actions. Empirically, model-based algorithms, able to enjoy a richer dataset, have shown better generalization capability (\cite{yuMOPOModelbasedOffline2020}, \cite{kidambiMOReLModelBasedOffline2021a}, \cite{yuCOMBOConservativeOffline2021}). \cite{yuCOMBOConservativeOffline2021} empirically demonstrates the hardness of performing uncertainty quantification under complex dynamics networks and environments, and proposes COMBO, a model-based counterpart of CQL that demands no explicit regularization. Specifically, COMBO employs an actor-critic scheme where the action value function is learned using both the offline dataset as well as synthetic rollouts (\cite{yuCOMBOConservativeOffline2021}).

Despite preliminary success of model-based control, we expect plenty of headroom to improve. The standard COMBO only minimizes the expectation of value function under the rollout-induced distribution but makes no guarantees on the state-action pairs that are covered by neither the dataset nor the rollouts. Although, during future iterations, the updated rollout policy may reach such states, restricting this earlier may result in a faster rate of convergence. Also, if the neighboring state-action pairs of those that are sampled from the rollout policy are not penalized, they are likely to be visited in testing time due to the randomness of the policy, which can be hazardous.

\begin{figure}[h]
	\centering
	\subfloat[]{{\includegraphics[width=0.5\textwidth]{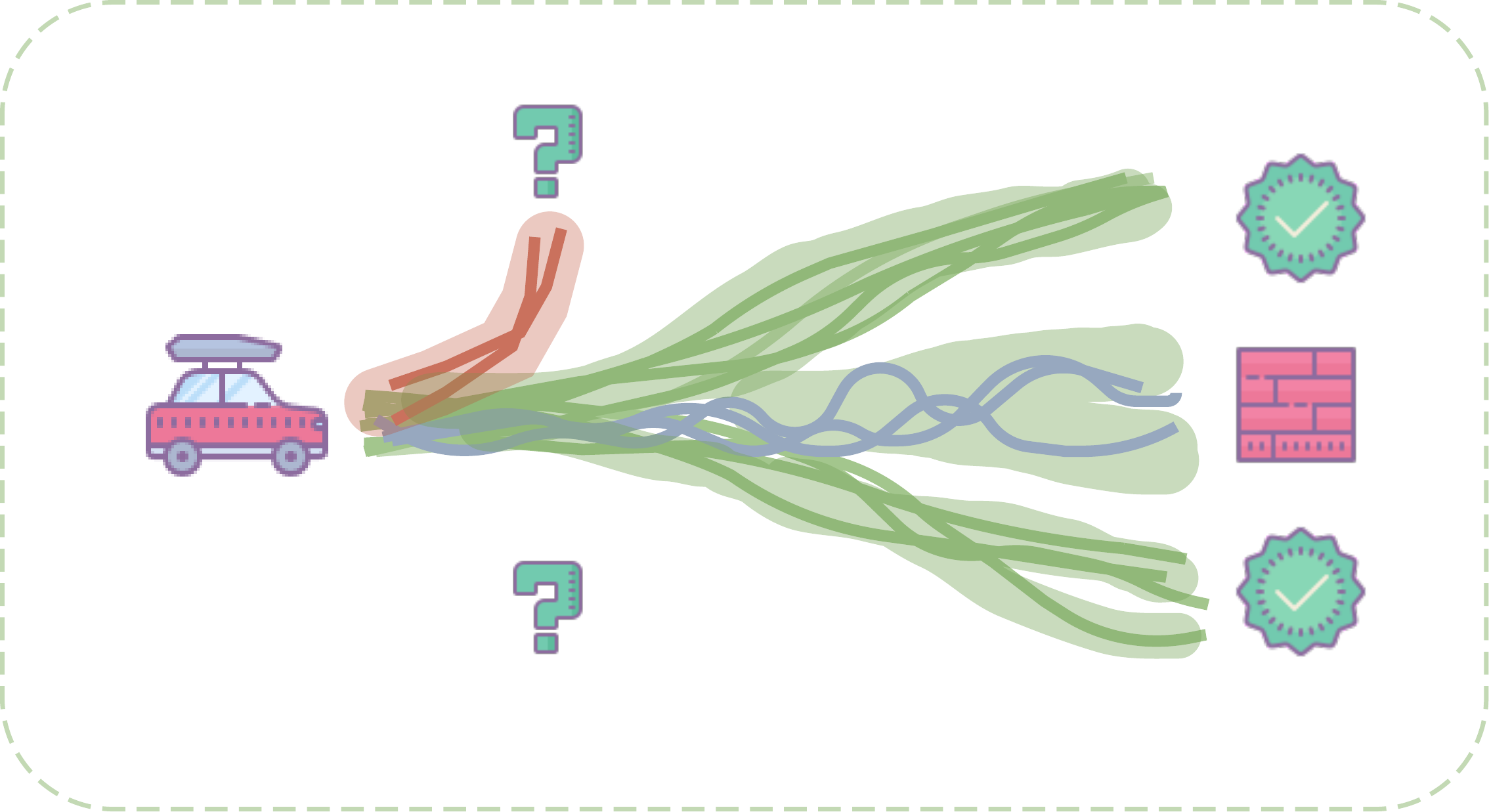}}}
	\qquad
	\subfloat[]{{\includegraphics[width=0.5\textwidth]{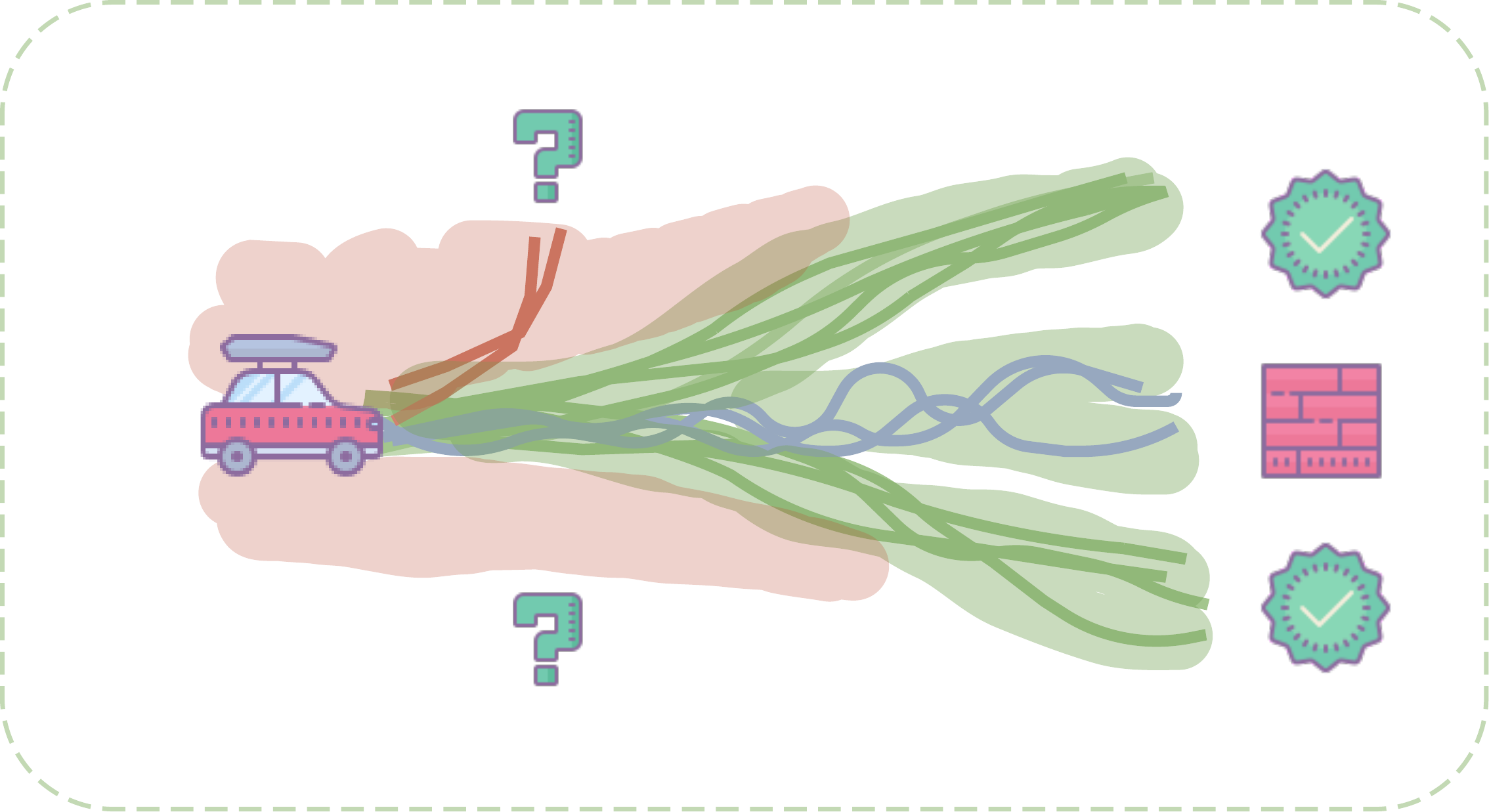} }}
	\caption{The penalty manifolds of (a) COMBO and (b) DROMO under a toy scenario. The green and blue lines represent the high-reward and low-reward states in the dataset support. The red lines represent the out-of-distribution state-action pairs carried out by the rollout agent. Note that for simplicity the in-distribution state-action pairs are absorded into the dataset suuport. Compared to COMBO, DROMO also generalizes the penalty to other out-of-distribution state-action pairs.
	}
	\label{combo-scene}
\end{figure}

The primary contribution of this work is a counterpart of COMBO that arises from theoretical results in distributionally robust optimization. Specifically, \cite{duchiVariancebasedRegularizationConvex2017} propose to add a variance regularization term on the loss function, which automatically balances bias and the variance. This is shown to improve out-of-sample testing-time performance (\cite{duchiVariancebasedRegularizationConvex2017}). In the sequel, we view adding the variance regularization as providing a convex surrogate of expectation value function under the empirical rollouts and allows us to penalize state-and-action beyond the coverage of both the rollouts and the dataset. Subsequently, this is illustrated through a toy example in Figure \ref{combo-scene}. In Section \ref{sec:dromo-lb}, we show that DROMO learns a lower bound of the expectation of Q-function obtained in the underlying MDP, and in Section \ref{sec:dromo-lin} and \ref{sec:dromo-nonlin}, we show that, under both linear and non-linear regimes, a simple and more conservative instantiation of DROMO holds similar guarantees.

\section{Preliminaries}\label{sec:prelim}
In this section, we review the standard fully-observed MDP setup, the problems that offline reinforcement learning tries to solve, and the model-based approaches.
\subsection{Episodic MDP and Policy Evaluation}

We consider episodic but infinite horizon, fully-observed Markov Decision Process, namely MDP (\cite{PUTERMAN1990331}). An MDP is a tuple of six elements: $\mM = (\mS,\mA,T,r,\mu,\gamma)$, where $\mS$ is the state space, $\mA$ is the action space, $T(\cdot|\ms,\ma): \mS \times \mA \to \mS$ is the transition probability density function conditioned on a given $(\ms,\ma)$-pair, $r(\ms,\ma): \mS\times\mA\to\R$ is the real-valued reward function for a given $(\ms,\ma)$-pair, $\mu(\ms)$ is the initial state distribution and $\gamma \in (0,1)$ is a discount factor . Define a policy $\pi(\ma|\ms): \mS \to \mA$ to be a probabilistic model conditioned on a given state $\ms$. Also, we denote the stationary distribution by playing a policy $\pi(a|s)$ under an MDP $\mM$ as $d^\pi_\mM(\bf s):= (1-\gamma)\sum^\infty_{t=0}\cal P{(\bf{s_t=s|\pi})}$. With an abuse of notation, we denote the visitation distribution of a given $(s,a)$-pair induced by a policy $\pi(a|s)$ running on MDP $\mM$ as $d^\pi_\mM(\ms,\ma):=d^\pi_\mM(s)\pi(\ma|\ms)$. The goal of reinforcement learning is to obtain a $\pi(a|s)$ that maximizes the expected cumulative discounted reward (i.e., return) over a target MDP $\mM$.

\begin{equation}
	\max_{\pi}\mJ(\mM,\pi):= \frac{1}{1-\gamma}\E_{\ms\sim d^\pi_\mM(s),\ \ma\sim\pi(\ma|\ms)}\left[r(\bf{\ms,\ma})\right]
	\label{eq:j-eval-def}
\end{equation}

For any policy $\pi(a|s)$ an MDP $\mM$ we define the state-value function $V^\pi_\mM$ as
\begin{equation}
	\label{eq:v-def}
	V^\pi(\ms):= \E_{s\sim d^\pi_\mM(s),\ a\sim\pi(a|s)}\left[\sum^\infty_{h=0} \gamma^h r{(\ms_h,\ma_h)|\ms_0 = \ms}\right]
\end{equation}
and the action-value function (Q-function) as
\begin{equation}
	\label{eq:q-def}
	Q^\pi(\ms,\ma):= \E_{s\sim d^\pi_\mM(s),\ a\sim\pi(a|s)}\left[\sum^\infty_{h=0} \gamma^h r{(\ms_h,\ma_h)|\ms_0 = \ms,\ma_0=\ma}\right].
\end{equation}

As a direct result of the definition in Equation \ref{eq:v-def} and \ref{eq:q-def}, we have the following equality.
\begin{equation}
	\begin{aligned}
		V^\pi(\ms) & = \sum_\ma \pi(\ma|\ms)Q(\ms,\ma)=\LL\langle Q(\ms,\cdot),\pi(\cdot|\ms)\RR\rangle_\mA\label{eq:v-bell} \\
	\end{aligned}
\end{equation}
where $\langle\,\cdot,\cdot\,\rangle_{\mA}$ is the inner product over the action space $\mA$.

Classic schema for solving MDPs include Q-learning and (\cite{sutton1999reinforcement}) and actor-critic (\cite{bertsekas1995neuro}). Q-learning approaches contract the Q-function by iteratively applying the Bellman optimality operator $\mB^* Q(\ms,\ma) :=r(\ms,\ma)+\gamma\E_{\ms'\sim T(\ms,\ma)}[\max_{\ma'}Q(\ms',\ma')]$, and then recover a greedy policy via an exact or approximate maximization scheme, such as CEM (\cite{kalashnikov2018scalable}).Actor-critic methods alternate update between two models: actor and critic. The critic updates iteratively contract the Q-function with Bellman expectation operator $\mB^\pi Q(\ms,\ma) :=r(\ms,\ma)+\gamma\E_{\ms'\sim T(\ms,\ma),\ma'\sim\pi(\ma'|\ms')}[Q(\ms',\ma')]$, and the actor updates the policy $\pi(\ma|\ms)$ to maximize the Q-estimate.
For the rest of the paper, we denote the policy class as $\Pi = \LL\{\pi_\theta: \theta \in \Theta\RR\}$, and the function class for Q-function as $\mathcal Q= \LL\{Q_\psi : \psi\in \Psi\RR\}$. However, under clear context, we simplify $Q_{\psi^k}$ as $Q^k$ and $Q^{\pi_{\theta^K}}_{\psi^K}$ as $Q^\pi$.

\subsection{Offline Data Collection}

We consider an offline(batch) settings, that is, the situation where the learner only has access to a dataset $\mD=\{(\ms,\ma,\ms',r)\}$, which are induced by an unknown behavior policy $\pi^\beta$ running on a ground-truth MDP $\widetilde \mM$. In short, the dataset is sampled from the $d^{\pi^\beta} =d^{\pi^\beta}(\ms)\pi^\beta(\ma|\ms)$. We also denote the empirical $\widetilde \mM$ to be the empirical MDP induced by the dataset $\mD$, which differs from the ground-truth MDP due to the heteroskedasticity and imbalancedness of the dataset and denote $d(\ms,\ma)$ to be the empirical surrogate of $d^{\pi^\beta}(\ms,\ma)$. The behavior can be approximated by maximum likelihood estimation $\pi^\beta\leftarrow\argmax_{\pi\in\Pi}\E_{s,a\sim\mD}\left[\pi(\ma|\ms)\right]$. Note that under offline setting, it is statistically hard to find an optimal policy over the ground-truth MDP (\cite{chenInformationTheoreticConsiderationsBatch2019} \cite{matsushimaDeploymentEfficientReinforcementLearning2020} \cite{dongProvableModelbasedNonlinear2021}). We assume that the distance between the dataset-induced MDP and the underlying MDP is bounded, which is a standard assumption in offline RL literature (\cite{kumarConservativeQLearningOffline2020}, \cite{yuCOMBOConservativeOffline2021}).

\begin{assumption}
	\label{ass:bound-q}
	\textit{Assume for all $\ms\in\mS\,,\ma\in\mA$, the following inequality holds with high probability ($\ge 1-\delta$)
		\begin{equation}
			|r(\ms,\ma)-r_{\widetilde \mM}(\ms,\ma)|\le \frac {C_{r,\delta}}{\sqrt{|\mD(\ms,\ma)|}}\,\ |T(\ms,\ma)-T_{\widetilde \mM}(\ms,\ma)|\le \frac{C_{T,\delta}}{\sqrt{|\mD(\ms,\ma)|}}
			\label{eq:bound-q}
		\end{equation}
		where $r(\cdot\,,\cdot)$ and $T(\cdot\,,\cdot)$ is the reward and dynamics of the underlying MDP $\mM$.
	}
\end{assumption}

Since the dataset $\mD$ does not cover all the $(\ms,\ma,\ms')$-transition pairs, the learner do not have access to the Bellman expectation operator $\mB^\pi$ of the underlying MDP, but a empirial surrogate $\hat \mB^\pi$ backed up by a single transition $(\ms,\ma,\ms)$, namely,
\begin{equation}
	\mB^\pi_{\widetilde \mM} Q(\ms,\ma) = r(\ms,\ma) + \gamma\E_{\ma'\sim\pi(\ma'|\ms')}[Q(\ms',\ma')]
	\label{eq:emp-bell}
\end{equation}
Following \cite{osband2016deep}, \cite{jaquesWayOffPolicyBatch2019}, and \cite{odonoghueVariationalBayesianReinforcement2019}, using the empirial Bellman operator to contract the Q-function is still viable and as shown in Appendix D.3 of \cite{kumarConservativeQLearningOffline2020}, the difference between is bounded: $|\mB^\pi Q(\ms,\ma)-\mB^\pi_{\widetilde \mM} Q(\ms,\ma)| \le \frac{C_{r,T,\delta}}{\sqrt{|\mD(\ms,\ma)|}}$.

\subsection{Model-based Reinforcement Learning}

Model-based methods learn an additional dynamics model and then use it to aid policy update. The dynamics model, denoted as $\hat T$, can also be trained under maximum likelihood estimate $\hat T = \argmin_{\phi \in \Phi}\E_{\ms,\ma,\ms'\sim\mD}[\log(T_\phi(\ms'|\ms,\ma))]$, where the function class for dynamics is given by $\mathfrak T: \LL\{\mathcal{N}(\mu_\phi(\ms,\ma),\Sigma_\phi(\ms,\ma)):\phi\in\Phi\RR\}$. If the reward function is unknown, we can also approximated the reward with the dynamics model $\hat T$ by concatenated onto the target state $\ms'$, and the bounds in Assumption \ref{ass:bound-q} still holds. We denote $\widehat \mM=\LL(\mS,\mA,\hat T, r_{\widehat \mM},\mu_0,\gamma\RR)$ as the MDP induced by the learned dynamics. This approach works orthogonal to all policy gradient algorithms, any of which can be used to recover a greedy policy on $\widehat \mM$. For any policy $\pi(a|s)$, we denote the occupancy measure of a certain $(s,a)$-pair induced by running $\pi(a|s)$ on the learned MDP $\widehat \mM$ as $\rho(s,a)=d^\pi_{\widehat \mM}(s)\pi(a|s)$. Here, we have access the empirical Bellman operator with respect to $\widehat \mM$,

\begin{equation}
	\mB^\pi_{\widehat \mM} Q(\ms,\ma) = r_{\widehat \mM}(\ms,\ma) + \gamma\E_{\ms'\sim\hat T(\ms,\ma)\ma'\sim\pi(\ma'|\ms')}[Q(\ms',\ma')]
	\label{eq:expectation-bell-hat}
\end{equation}

However, directly applying approximate dynamic programming algorithms over fails Equation \ref{eq:expectation-bell-hat} theoretically and practically, due to the distribution shift between the dataset $\mD$ and the learned MDP $\widehat \mM$ (\cite{ross2012agnostic}, \cite{kidambiMOReLModelBasedOffline2021a}). Methods for soothing this issue include MORel (\cite{kidambiMOReLModelBasedOffline2021a}), MOPO (\cite{yuMOPOModelbasedOffline2020}), and COMBO (\cite{yuCOMBOConservativeOffline2021}). MOReL and MOPO use uncertainty quantifier as a penalty term and optimize over the lower bound of the policy evaluation function. Uncertainty quantification algorithm like bootstrap ensembles (\cite{osband2018randomized}, \cite{azizzadenesheli2018efficient}, \cite{lowreyPlanOnlineLearn2019}), we can estimate uncertainty $u(\ms,\ma)$ in the output of dynamics model $\hat T$ given an $(\ms,\ma)$-pair. MOPO use this uncertainty measure to subtract the corresponding reward of all state-action pairs $\bar r(s,a) = \hat r(s,a)-\lambda u(s,a)$, and run a policy gradient algorithm on the MDP $\bar \mM=\LL(\mS,\mA,\hat T,\bar r,\mu,\gamma\RR)$ constructed upon the truncated reward. However, uncertainty estimation based on variance is not sensitive enough and does not capture the shape of the support of dataset \cite{yuCOMBOConservativeOffline2021}. COMBO adopts the framework of CQL to inject pessimism into the Q-function and policy search without explicitly penalizing uncertainty. Our work is largely based on COMBO, which alternates between the following actor and critic update:
\begin{itemize}
	\item \textbf{Critic update.}
	      The Q-function is iteratively updated by solving the following minimization problem
	      \begin{equation}
		      \begin{aligned}
			      \hat Q^{k+1}\leftarrow \argmin_Q \beta\LL(\E_{\ms,\ma\sim\rho(\ms,\ma)}[Q(\ms,\ma)]-\E_{\ms,\ma\sim \mD}[Q(\ms,\ma)]\RR) \\
			      +\frac 1 2 \E_{\ms,\ma,\ms'\sim d_f}\LL[\LL(Q(s,a)-\hat \mB^\pi Q(s,a)\RR)^2\RR]
		      \end{aligned}
		      \label{eq:combo-critic-update}
	      \end{equation}
	      where $\rho(\ms,\ma)=d^\pi_{\widetilde \mM}(\ms)\pi(\ma|\ms)$ is the state-action occupancy distribution induced by a policy $\pi$ and $\widehat \mM$, and $d_f(\ms,\ma)$ is the $f$-interpolate of the policy-induced and dataset-induced state-action occupancy distribution, defined as
	      \begin{equation}
		      d_f(\ms,\ma) = f\rho(s,a) + (1-f)d(s,a)
		      \label{eq:f-int}
	      \end{equation}
	      where $d(s,a)$ is the empirical version of dataset $d^{\pi^\beta}_{\widetilde \mM}(s)\pi^\beta(a|s)$. COMBO penalizes Q-function for $(\ms,\ma)$-pairs that are rare or unseen in the dataset $\mD$ (\cite{yuCOMBOConservativeOffline2021}). Unlike CQL, over-estimation do occur, but only for $(\ms,\ma)$-pairs that are sampled from the dataset.
	\item \textbf{Actor update.} After learning a conservative critic, $\hat Q^{k+1}$, we update the actor(policy) as
	      \begin{equation}
		      \pi^{k+1} \leftarrow\argmax_{\pi\in\Pi} \E_{\ms\sim\rho(s),\ms\sim\pi}\left[\hat Q^{k+1}(\ms,\ma)\right]
	      \end{equation}
	      where $\rho(s)$ is the stationary distribution of $\rho(s,a)$. In the actor-critic schema, both \texttt{argmax} and \texttt{argmin} can be approximated by sufficiently many steps of gradient descent.
\end{itemize}

\section{Distributionally Robust Offline Model-based Optimization}\label{dromo-alg}

In this section, we present an offline arctor-critic algorithm \textit{\textbf{d}istribution \textbf{r}obust \textbf{o}ffline \textbf{m}odel-based policy \textbf{o}ptimization} (DROMO). Specifically, we leverage the recent results in distributionally robust optimization (DRO) by (\cite{duchiVariancebasedRegularizationConvex2017}). We first review the notion of variance-based robust optimization: for a general loss function $\ell:\Psi\times X\to\R$ and the minimization problem,
\begin{equation}
	\minimize_{\psi\in\Psi}R(\psi)=\E[\ell(X;\psi)]=\int d\mu(X)\ell(\psi,X),
\end{equation}
given training data $X\in\{X_i\}^N_{i=1}$ drawn i.i.d from a distribution $\mu(X)$,
we write the robustly regularized risk as follows,
\begin{equation}
	R_{\rob}(\psi)= \sup_{\hat \mu}\left\{\E_{\hat \mu}[\ell(X;\psi)]:\chis{\hat \mu}{\mu}\le\frac\alpha n\right\}
\end{equation}
where $\chis{\cdot}{\cdot}$ is the $\chi^2$-divergence.
\begin{theorem}[Truncated version of Theorem 1, \cite{duchiVariancebasedRegularizationConvex2017}]

	\textit{Let $Z$ be a random variable taking values in a bounded distribution $\mu$. Let $s^2_n = \E_{\mu}[Z^2]$ be the sample variance of $Z$. For any fixed $\alpha>0$,
		\begin{displaymath}
			\sup_{\hat \mu}\left\{\E_{\hat \mu}[Z]:\chis{\hat \mu}{\mu}\le\frac\alpha n\right\}
			\le \E_{\mu}[Z]+ \sqrt{\frac \alpha n s^2_n}
		\end{displaymath}
		for a sufficiently large $n$, the equality holds with high probability.}
\end{theorem}

As an immediate corollary, we can see that  $\E_{\mu}[z]+ \sqrt{\frac \alpha n s^2_n}$ upper-bounds the robustly regularized risk and served as a convex surrogate for the empirical risk, which also enjoys a faster rate of convergence because of strong convexity (\cite{duchiVariancebasedRegularizationConvex2017}). This gives rise to the design of a more robust conservative actor-critic update rule:

\textbf{Distributionally robust critic update.} We can minimize the upper bound of the expected Q-estimate evaluated on the distribution induced by $\pi$ and the learned MDP $\hat \mM$, and update the Q-function as the following minimization task,

\begin{equation}
	\label{eq:dr-combo}
	\begin{aligned}
		Q^{k+1}(\ms,\ma) \leftarrow & \argmin_Q \frac 1 2 \E_{\ms,\ma,\ms'\sim d_f}\left[\left(Q(\ms,\ma)-{\hat \mB^\pi}Q^k)(\ms,\ma)\right)^2\right]    \\
		+                           & \alpha \E_{\ms\sim d_f} \left[\sqrt{\frac {\text{var}_{\pi^f}(Q(\ms,\ma))}{|\mD|d^{\pi^\beta}(\ms)}}\right]        \\
		+                           & \beta \left(\mathbb E_{\ms,\ma\sim \rho(\ms,\ma)}[Q(\ms,\ma)]-\mathbb E_{\ms,\ma\sim\mathcal D}[Q(\ms,\ma)]\right) \\
	\end{aligned}
\end{equation}
where $d_f(\ms,\ma)$ is the f-interpolation of rollout-induced stationary distribution $\rho(\ms,\ma)$ and the dataset-induced stationary distribution $d(\ms,\ma)$. In the variance term, the states $\ms$ are sampled from the state marginal $d_f(\ms)$ of the $d_f(\ms,\ma)$, and the variance is taken over the f-interpolated policy $\pi^f(\ma|\ms)=f\pi^\beta(\ma|\ms)+(1-f)\pi(\ma|\ms)$.

The insight behind Equation \ref{eq:dr-combo} is that DROMO penalizes the critic for its standard deviation across actions induced by the f-interpolated policy $\pi^f$, and for $(\ms,\ma)$-pairs that are sampled from the rollouts but are out of the support of the dataset-induced stationary distribution. On the other hand, the Q-function penalizes the Q-function less for state-action pairs that come from the underlying MDP.

\textbf{Distributionally robust actor update.} Now that we have a distirbutionally robust critic, we update the actor as
\begin{equation}
	\label{eq:dromo-actor}
	\argmax_{\pi\in\Pi}\E_{\rho{\ms}.\pi(\ma|\ms)}[Q(\ms,\ma)]+\cal R(\pi)
\end{equation}
where, following the design of other policy gradient algorithms, we approximate the \texttt{argmax} with a few steps of gradient descent, and $\cal R(\pi)$ is an optional regularizer. We use $\cal R(\pi)$ to prevent the policy from degenerating, and choices include entropy regularization term of SAC (\cite{haarnojaSoftActorCriticOffPolicy2018}), Stein variational gradient descent (\cite{liuSteinVariationalPolicy2017}), and approximated natural gradient gradient descent (\cite{wuScalableTrustregionMethod2017}).

Also, we can use target Q-functions and policy to enhance the stability during training (\cite{lillicrapContinuousControlDeep2019}, \cite{mnih2015human}). Then the Bellman operator is applied over the target Q-function instead.
\begin{algorithm}
	\floatname{algorithm}{Algorithm}
	\caption{Distributionally Robust Offline Model-based Policy Optimization (DROMO)}
	\label{alg:dromo}
	\begin{algorithmic}[1]
		\Require Offline Dataset $\mD$, target network update rate $\tau$, initialized actor $\pi_{\theta}$, critic $Q_{\psi}$, the dynamics $T_\phi$, a target policy $\pi_{\theta'}$ with $\theta'\leftarrow \theta$, and target critic $Q_{\psi'}$ with $\psi' \leftarrow \psi$.
		\State $\hat T \leftarrow$ Train the probabilistic model $T_\phi(\ms',r|\ms,\ma)=\mathcal{N}(\mu_\phi(\ms,\ma),\Sigma_\phi(\ms,\ma))$ under the dataset $\mD$;
		\State Initialize the replay buffer $\mD_{\text{model}}\leftarrow \emptyset$
		\For{iteration $i=0,1,2,\cdots$}
		\State Sample initial state from the dataset $\mD$, and use the target policy $\pi^{i}_{\theta'}$ and the learned dynamics $\hat T_{\phi}$ to perform $K$ trajectories of $H$-episode rollouts, and update the buffer as $\mD_{\text{model}}\leftarrow\{(\ms^\tau,\ma^\tau,\ms'^\tau,r^\tau)\}^{K,H}_{\tau,h=1}\cup\mD_{\text{model}}$;
		\State Sample from $\mD\cup\mD_{\text{model}}$, and update the current policy evaluation by solving Equation \ref{eq:dr-combo}, obtaining $\psi^{i+1}$
		\State Optimize the actor $\pi_\theta$ by solving Equation \ref{eq:dromo-actor} under the new critic $Q^{i+1}_\psi$ and state marginal $d_f$, obtaining $\theta^{i+1}$
		\State Update target networks: $(\theta')^{i+1}\leftarrow\tau\theta^{i+1}+\tau(\theta')^{i};\, (\psi')^{i+1}=\tau\psi^{i+1}+(1-\tau)(\psi')^i$
		\EndFor
	\end{algorithmic}
\end{algorithm}

\section{Theoretical Analysis of DROMO}\label{sec:dr-combo}
In this section, we demonstrate the theoretical analysis of our approach. We show that DROMO optimizes a lower bound
on the expected cumulative rewards of the learned policy. For a single $(\ma|\ms)$ pair. We show that, modulo sampling error, the lower bound approaches the ground truth if the f-interpolated policy $\pi^f(\ma|\ms)$ is very confident in choosing the action $\ma$ in question and (\ms,\ma) is the support of both the rollouts and the dataset. As for case studies on concrete settings, we study the instantiations of DROMO under linear and neural tangent kernel regime. We deferred the proofs to Appendix \ref{theorem:proofs}.

\subsection{Lower-Bound}
\label{sec:dromo-lb}
Under finite state and action space, i.e., $|\mS|\,,|\mA|<\infty$, we can use approximate dynamic programming to update the Q-function $\hat Q^k$ at iteration $k$.
\begin{lemma}[distribution robust f-interpolation update]\label{q-update-lemma}
	\textit{
		For any \(f\in [0,1]\), and any given \(\rho(\ms,\ma) \in \Delta^{|\mathcal
			S||\mathcal A|}\), let $\pi^f$ be an f-interpolation of $\pi^\beta$ and $\pi$, and \(d_f\) be an f-interpolation of \(\rho\) and \(\mathcal
		D\). Let $\lambda_{\alpha,\pi,f}(\ms,\ma)=\alpha(1-\pi^f(a|s))$, for a given iteration \(k\) of Equation \ref{eq:dr-combo}, define the Q-update of a
		certain \((\ms,\ma)\)-pair as,
		\begin{equation}\label{q-update}
			\begin{aligned}
				\hat Q^{k+1}(\ms,\ma) & =\left(1-\frac
				{\lambda_{\alpha,\pi,f}(\ms,\ma)}{\lambda_{\alpha,\pi,f}(\ms,\ma)+\COne}\right) \\
				                      & \left\{(\hat \mB \hat
				Q^{k})(\ms,\ma)- \beta \left[\frac
					{\rho(\ms,\ma)-d(\ms,\ma)}{d_f(\ms,\ma)}\right]\right\}                         \\
			\end{aligned}
		\end{equation}
	}
\end{lemma}
For a single batch, we see that the Q-estimate is truncated not only for OOD actions, but also for $(s,a)$ pair that the $f$-interpolated policy $\pi(\ma,\ms)$ is not fully confident about.
Next, we show that, for any policy $\pi(\ma|\ms)$, the Q-function function $\hat Q^{k+1}$ lower-bounds by the function the actual Q-function with high probability with a sufficiently large non-negative $\beta$ that can overrule the sampling error. This is possible because the sampling error is unbiased in expectation and independent of $\beta$. The sampling error is also bounded due to Assumption \ref{ass:bound-q}, which is standard in literature (\cite{laroche2019safe}, \cite{kumarConservativeQLearningOffline2020}, \cite{yuCOMBOConservativeOffline2021}). We now show that our method optimizes over a lower bound similar to that of COMBO.
\begin{theorem}[Asymptotic lower bound]\label{asymp}
	\textit{
		Let $P^\pi$ be the hadamard product of the dynamics $P$ and a given policy
		$\pi$ in the actual MDP and let $S^\pi := (1-\gamma P^\pi)c_s$, with $c_s$
		being a sufficiently large positve constant.For any $\pi(a|s)$, the
		$Q$-function obtained by iteratively playing equation \ref{eq:dr-combo} is, with $\hat
			\mB^\pi = f\mB^\pi_{\widetilde \mM} + (1-f)\mB^\pi_{\hat \mM}$, with
		probability at least $1-\delta$, the resulting $\hat Q^\pi$ satisfies,
		\begin{displaymath}
			\forall \ms,\ma,\ \hat Q^\pi(\ms,\ma) \le\left(1-\frac
			{\lambda_{\alpha,\pi,f}(\ms,\ma)}{\lambda_{\alpha,\pi,f}(\ms,\ma)+\sqrt{|\mD(\ms)|\var_{\pi^f}(\hat
					Q^\pi)}}\right)(Q^\pi(\ms,\ma)- \beta \xi_1 + f \xi_2 +
			(1-f)\xi_3)
		\end{displaymath}
		where $\xi_1$, $\xi_2$ and $\xi_3$ is given by:
		\begin{equation*}
			\begin{aligned}
				\xi_1(\ms,\ma) & = \left[\frac 1 c_s S^\pi \left[\frac
						{\rho-d}{d_f}\right]\right](\ms,\ma),                           \\
				\xi_2(\ms,\ma) & = \left[ S^\pi\left[|R-R_{\widehat \mM}|+\frac
				{2\gamma R_{\max}}{1-\gamma}\TV(T,T_{\widehat \mM}) \right]
				\right](\ms,\ma),                                               \\
				\xi_3(\ms,\ma) & = \left[S^\pi \left[\frac
				{C_{r,T,\delta}R_{\max}}{(1-\gamma)\sqrt{|\mD|}} \right]
				\right](\ms,\ma).                                               \\
			\end{aligned}
		\end{equation*}
	}
\end{theorem}
\begin{corollary}[Informal]
	\label{corollary-smaller}
	\textit{
		For sufficiently large $\beta$ and an initial state distribution $\mu(s)$,
		we have $\E_{s\sim \mu(s), a\sim \pi(a|s)}[\hat Q(\ms,\ma)] = \E_{s\sim \mu(s), a\sim \pi(a|s)}[Q(\ms,\ma)]$ with probability at least $1-\delta$.
	}
\end{corollary}

Note that our bound is not as tight as that of COMBO, but the new Q-function $\hat Q^{k+1}$ is rescaled for every $(s,a)$ according to the marginal $d^{\pi^\beta}(s)$ of the state $\ms$ and the conditional probability of $f$-interpolated policy $\pi(\ma,\ms)$ about the action $\ma$ under the state $\ms$. We also show that DROMO has a gap-expanding property similar to that of CQL (\cite{kumarConservativeQLearningOffline2020}). This is also evidence of DROMO pushing up the state-action pairs in the support of the dataset and pushing down the rest.

\begin{theorem}[Gap-expanding property of DROMO]\label{thrm:gap-expanding}
	\textit{
		At any iteration $k$, DROBO expands the difference of $Q$-estimate of $d$
		and $\rho$, i.e. for sufficiently large $\beta$ and with high probability ($\ge1-\delta$) we have,
	}
	\begin{equation}
		\E_{\ms,\ma \sim d(\ms,\ma)}[\hat Q^k(\ms,\ma)] - \E_{\ms,\ma \sim
			\rho(\ms,\ma)}[\hat Q^k(\ms,\ma)] > \E_{\ms,\ma \sim
			d(\ms,\ma)}[Q^k(\ms,\ma)] - \E_{\ms,\ma \sim \rho(\ms,\ma)}[Q^k(\ms,\ma)]
	\end{equation}
\end{theorem}

Offline model-free policy-constraint methods that do not have such gap-expanding property, such as BEAR, tend to suffer from an unlearning effect similar to the one suffered by running vanilla off-policy arctor-critic algorithms under offline regime (\cite{kumarStabilizingOffPolicyQLearning2019}, \cite{levineOfflineReinforcementLearning2020}). Without policy constraint, the erroneous Q-function may lead to actor to OOD actions. However, occasionally policy constraint also fails to prevent the actor opt for OOD actions that induced invalidly high Q-values (\cite{kumarConservativeQLearningOffline2020}). DROMO addresses this issue by pushing up all the state-action pairs that are in the support of dataset while pushing down the OOD actions in the $\frac{\alpha^2(1-f)^f}{\beta^2|\mD(\ms)|}$-neighborhood of $\pi_f(\ma|\ms)$ with respect to $\chi^2$-divergence for all states in the marginal $d_f(s)$. Thus, the policy will optimize over a Q-function in favor of in-distribution, robust actions.
\subsection{DROMO with Linear Function Approximation}
\label{sec:dromo-lin}
In the following two sections, we present additional theoretical properties of different instantiations of DROMO. For the ease of presentation, we use a stricter version of DROMO, whose critic updates as follows,
\begin{equation}
	\label{eq:dr-combo-alt}
	\begin{aligned}
		Q^{k+1}(\ms,\ma) \leftarrow & \argmin_Q \frac 1 2 \E_{\ms,\ma,\ms'\sim d_f}\left[\left(Q(\ms,\ma)-\mathcal {\hat B^\pi}Q^k)(\ms,\ma)\right)^2\right] \\
		+                           & \alpha\sqrt{\frac {\text{var}_{d^f}(Q(\ms,\ma))}{n}}                                                                   \\
		+                           & \beta \left(\mathbb E_{\ms,\ma\sim \rho(\ms,\ma)}[Q(\ms,\ma)]-\mathbb E_{\ms,\ma\sim\mathcal D}[Q(\ms,\ma)]\right).    \\
	\end{aligned}
\end{equation}
The strict DROMO penalizes the variance of Q-function only across actions but also state-action pairs from $d_f(\ms,\ma)$. We start with the linear case, where the expected reward is a linear map from a state-action feature space.
\begin{theorem}\label{linear-theo}
	\textit{Let q-function be a linear approximator parameterized by $\omega \in \R^{\dim(\mat F(\ms,\ma))}$, $Q_\omega(s,a):= \mat{F(s,a)}\omega$ given a state-action feature $(1,0)$-tensor $\mat F(\ms,\ma)$, let $\mat U:=diag(\rho(\ms,\ma))$ be the matrix with diagonal data density entries of the synthetic rollouts and $\mat V:=diag(d(\ms,\ma))$ be the matrix with diagonal data density of the dataset, and assume both $\mat F^T\mat U \mat F$ and $\mat F^T \mat V \mat F$ are invertible. let $\mat D_f:= f\mat V+(1-f)\mat U$ be the $f$-interpolation of $\mat U$ and $\mat V$. then, under the f-interpolated distribution, the expectation of q-function obtained by playing equation \ref{eq:dr-combo-alt} for one gradient step at iteration $k+1$ lower bounds the expectation of the tabular function iterate if:
	\begin{displaymath}
		\beta^k \ge \frac{(\mu\cdot\pi)^T\left[\mat{F}\left(\mat F^T\mat{D_fF}\right)^{-1}\mat{F}^T-\mat{Id}_{|\mS||\mA|}\right] \left(\hat \mB^\pi \hat Q^k\right)+\xi_{r,T,\delta}}{\left(\mu\cdot\pi\right)^T\mat{F}\left(\mat F^T \mat {D_fF}\right)^{-1}\mat F^T(\rho-d)}-\frac{\alpha^k n^{-\frac 1 2}\cov_{d_f}(\mat{F,F\omega}^k)}{\mat F^T(\rho-d)\CThr}
	\end{displaymath}
	where $\xi_{r,T,\delta} = (1-f)\left[|r-r_{\widehat \mM}|+\frac {2\gamma R_{\max}}{1-\gamma}\TV(T,\hat T)\right] + \frac{fC_{r,T,\delta}R_{\max}}{(1-\gamma)\sqrt{|\mD|}}$.}
\end{theorem}

\subsection{DROMO with non-linear function approximation}
\label{sec:dromo-nonlin}
Next, we generalize our result in Theorem \ref{linear-theo} to the non-linear case, such as feed-forward neural networks, under the framework of neural tangent kernel. The neural tangent kernel (NTK) approach is standard in optimization theory literature, reducing a sublinear problem to a linear one and giving insights into the loss landscape in the neighborhood of an initialization (\cite{jacotNeuralTangentKernel2020}, \cite{valkofinitetime}, \cite{zhou2020neural}). The core idea is to assume that, for an initialization $\psi_0$, there exists a nice neighborhood in which Equation \ref{eq:dr-combo-alt} is convex, and a global minimum of it in the set. In short, the Equation \ref{eq:dr-combo-alt} converges to a global minimum under our chosen initialization $\psi_0$.
\begin{theorem}\label{non-linear-theo}
	\textit{Let q-function be a non-linear approximator parameterized by $\psi$, $Q_\psi(s,a)$, let $\mat U:=diag(\rho(s,a))$ be the matrix with diagonal data density entries of the synthetic rollouts and $\mat V:=diag(d(s,a))$ be the matrix with diagonal data density of the dataset, and assume both $(\gQ)^T\mat U \gQ$ and $(\gQ)^T\mat V\gQ$ are invertible for all $k\in \Z_{+}$. let $\mat D_f:= f\mat V+(1-f)\mat U$ be the $f$-interpolation of $\mat U$ and $\mat V$. then, under the f-interpolated distribution, the expectation of q-function obtained by playing equation \ref{eq:dr-combo-alt} for one gradient step at oteratopm $k+1$ lower bounds the expectation of the tabular function iterate if}\footnote{Note that, since $\mat{D_f}$ is symmetric, $\mat{D_f}^{\odot 2} = \mat{D_f}^T\mat{D_f}$}:
	\begin{displaymath}
		\begin{aligned}
			\beta^k & \ge \xi_{r,T,\delta}\mat{(U-V)^{-1}}\mat{D_f} - \eta\alpha^k(\mu(s)\cdot\pi(a|s))^T\sqrt{1/n}\frac{\left(\mat{D_f - D_f^{\odot 2}}\right)\mat K^k \hat Q^k}{\CTwo}
		\end{aligned}
	\end{displaymath}
	\textit{where $\xi_{r,T,\delta}$ is given by $(1-f)\left[|r-r_{\widehat \mM}|
		+\frac {2\gamma R_{\max}}{1-\gamma}\TV(T,\hat T)\right] + \frac{fC_{r,T,\delta}R_{\max}}{(1-\gamma)\sqrt{|\mD|}}$.}
\end{theorem}



\section{Related Work}

Reinforcement learning (RL) algorithms are well known for their abilities to acquire good plannings fromd online trial-and-errors (\cite{bartoNeuronlikeAdaptiveElements1983}, \cite{sutton1999reinforcement}). However, such an online learning process induced large sample complexity (\cite{mnih2016asynchronous}, \cite{schulmanTrustRegionPolicy2015}, \cite{schulmanProximalPolicyOptimization2017}), raising the cost of training agents and risks under safety-critical scenarios (\cite{thomas2015safe}).

This works adds to the vast field of study of RL under offline settings, which is a way of tackling the challenges on the costly side of online RL (\cite{lange2012batch}, \cite{levineOfflineReinforcementLearning2020}). Listed below is prior works in offline RL and off-policy policy gradient algorithms, and some brief discussions on the differences between these works and ours.

\subsection{Off-policy Evaluation}
In off-policy evaluation, importance sampling (\cite{degrisOffPolicyActorCritic2013}, \cite{schulmanTrustRegionPolicy2015}) is used to (i) estimate the expected cumulative reward under the current policy and (ii) estimate the gradient of this expected cumulative reward. However, importance sampling itself already imposes great variance, and the variance tends to accumulate exponentially in sequential settings (\cite{levineOfflineReinforcementLearning2020}) Marginal importance sampling (\cite{liu2018breaking}, \cite{gelada2019off}, \cite{nachum2019dualdice}, \cite{zhou2020neural}) has been proposed to alleviate this issue by avoiding the multiplication of importance weights over time steps. However, the fundamental issue remains: (i) the importance weight is too variant to be usable in the estimation of cumulative rewards or gradients (\cite{levineOfflineReinforcementLearning2020}) and (ii) the learned policy suffers from OOD actions. In contrast, via implicit uncertainty quantification similar to CQL (\cite{kumarConservativeQLearningOffline2020}) and COMBO (\cite{yuCOMBOConservativeOffline2021}) and as shown in Theorem \ref{thrm:gap-expanding}, DROMO expands the gap between the action value under in-distribution actions and OOD actions.
\subsection{Model-free offline reinforcement learning}

Prior model-free RL algorithms have focused on regularizing the policy improvement process so as to explicitly restrict the divergence between the current policy and behavior policy.
In the actor update objective, various methods concern adding a penalty of different $\phi$-divergences between the current policy and the behavior policy, $D_\phi(\pi\|\pi^\beta)$. Notable choices include direct state-action constraint (\cite{fujimotoOffPolicyDeepReinforcement2019}), KL-divergence (\cite{jaquesWayOffPolicyBatch2019}, \cite{zhouPLASLatentAction2020}), Wasserstein(\cite{wuBehaviorRegularizedOffline2019}), MMD (\cite{kumarConservativeQLearningOffline2020}). The choices of penalty terms sometimes extend beyond $\phi$-divergences, such as uncertainty measures (\cite{agarwalOptimisticPerspectiveOffline2020}, \cite{levineOfflineReinforcementLearning2020}), as well as other Integral Probability Metrics (IPMs) (see \cite{sriperumbudurIntegralProbabilityMetrics2009} for more details). Different from the aforementioned approaches, DROMO utilizes a predictive model that allows the Q-function to learn in a richer dataset, and instead of restricting the policy directly, DROMO regularizes the policy evaluation to give lower scores on OOD actions, thus achieving a less conservative policy update.

\subsection{Model-based offline reinforcement learning}

In this work, we focus on the model-based control aspect of offline reinforcement learning. Details related to the general framework of model-based offline RL are supplied in Section \ref{sec:prelim}. Many model-based methods can do well under standard off-policy settings (\cite{sutton1991dyna}, \cite{watterEmbedControlLocally2015}, \cite{jannerWhenTrustYour2019}, \cite{hafnerLearningLatentDynamics2019}, \cite{zhangSOLARDeepStructured2019}). \cite{yuMOPOModelbasedOffline2020} also empirically show that, even without any extra penalty or regularization, MBPO (\cite{jannerWhenTrustYour2019}) can perform reasonably well under offline regime. \cite{berkenkamp2018safe} propose a safe-region-constraint method that guarantees Lyapunov stability. \cite{rhinehartDeepImitativeModels2019} propose deep imitative models (DIMs), which learn a normalizing flow to predict the future trajectories and use this model to prevent the actor to take trajectories that significantly diverge from the training data. On the other hand, MoREL (\cite{kidambiMOReLModelBasedOffline2021a}), MOPO (\cite{yuMOPOModelbasedOffline2020}), and COMBO (\cite{yuCOMBOConservativeOffline2021}) learn conservative value estimates with analytic bounds on performance, either explicitly or implicitly. This work follows the conservative model-based RL principle that the critic penalizes the actor for taking actions where learned dynamics model is probably erroneous. Compared to the COMBO, on which this work is based, DROMO utilizes a variance regularization term, which is used in distributionally robust optimization in \cite{duchiVariancebasedRegularizationConvex2017} and obtains a more robust expectation of the Q-value.


\section{Conclusion and Discussion}

In this work, we propose distributionally robust offline model-based policy optimization (DROMO), a model-based offline reinforcement learning algorithm that penalizes the Q-value for OOD state-action pairs as well as its own variance. In particular, compared to the previous work we invoke an identity from \cite{duchiVariancebasedRegularizationConvex2017}'s work on distributionally robust optimization (DRO) to obtain a robust expectation of the Q-value evaluated under the rollout-induced distribution. Despite the theoretical advantage of DROMO, our future work will be empirically evaulating the performance of DROMO.
In addition, DROMO, with variance regularization, may exhibit a lower sample complexity as well as a faster rate compared to COMBO \cite{yuCOMBOConservativeOffline2021}, which we also plan to show in the future. There is also a number of directions for future work. One of the interesting avenues for future research to inject or enforce pessimism by applying DRO chance constraints under different ambiguity sets, such as Wasserstein. Also, as we are able to present an offline hyperparameter tunning scheme for MOPO in Appendix \ref{auto-temp}, it is difficult to use the same method to automatically control $f$, the interpolation factor of rollouts and the dataset. Also, by writing the following remark, which follows from Theorem \ref{asymp}, we suspect that Lipschitz regularization can also improve Model-based offline RL.

\begin{remark}
	\textit{Assume the Q function evaluated under the underlying MDP and a certain policy $\pi$ is $\kappa_Q$-Lipschitz with respect to some $L^p$ norm $\|\cdot\|_{L^p}$, i.e.$\forall \ms,\ms' \in \mS\,, \ma,\ma' \in \mA$
	\begin{displaymath}
		|Q^\pi(\ms,\ma)-Q^\pi(\ms',\ma')| \le \kappa_Q(\|\ms-\ms'\|_{L^p}+\|\ma- \ma'\|_{L^p})
	\end{displaymath}
	then the Q-value $\hat Q^\pi$ learned by DROMO satisfies,
	\begin{displaymath}
		|\hat Q^\pi(\ms,\ma)-\hat Q^\pi(\ms',\ma')| \lesssim \left(\kappa_Q-\frac
		{\kappa_Q\lambda_{\alpha,\pi,f}(\ms,\ma)}{\lambda_{\alpha,\pi,f}(\ms,\ma)+\sqrt{|\mD(\ms)|\var_{\pi^f}(\hat
				Q^\pi)}}\right)(\|\ms-\ms'\|_{L^p}+\|\ma- \ma'\|_{L^p})
	\end{displaymath}
	}
\end{remark}

\subsection*{Acknowledgments}
We sincerely thank our parents for their sponsorship; Xinyi Liu for setting us off; Xiaohua Xie for suggestions on the future experiments; Hengzhi Wang and Zihao Liu for helpful discussion; Daoyuan Ren, Zexian Liang, and Affiliated High School of South China Normal University for their acknowledgment of this project.

\bibliography{refs.bib}
\bibliographystyle{icml2019}
\appendix
\onecolumn
\section{Additional notes on dynamics calibration}

We are trying to evaluate the policy $\pi$ given any state $s$ using the typical value iteration framework
\begin{equation}
	\label{eq:cal1}
	V^\pi(\ms) = \E_{\ma \sim \pi(\cdot|\ms)}[r(\ms,\ma)]+\gamma\E_{\ma \sim
		\pi(\cdot|\ms),\ms' \sim \hat T(\ms,\ma)}[V^\pi(\ms')]
\end{equation}
For the given policy $\pi$ and an MDP, there eixsts a stationary distribution
$d^\pi_\mM(\ms)$, defined as the probability of running into some particular
state $\ms$ in infinitely long horizons under the given MDP. We evaluate the
policy of $\pi$ as an expectation of V with respect to states sampled from the
stationary distribution, i.e.,
\begin{equation}
	\mJ(\pi) = \E_{\ms\sim d^\pi_\mM(\ms)}[V^\pi(\ms)]
\end{equation}
It follows that,
\begin{equation*}
	\begin{aligned}
		\mathcal J(\pi) & = \mathbb E_{\ms\sim d^\pi_\mM}[V_\pi(\ms)]                                                                                                      \\
		                & = \mathbb E_{\ms\sim d^\pi_\mM}[R(\ms)] + \gamma\mathbb E_{\ms\sim\rho^\pi}\mathbb E_{a\sim\pi(\cdot|s)}\mathbb E_{\ms'\sim\hat T(s,a)}[V(\ms')] \\
		                & = \mathbb E_{\ms\sim d^\pi_\mM}[R(\ms)] + \gamma\mathbb E_{((\ms,\ma),\ms')\sim\mathbb P}[V(\ms')]                                               \\
		                & = \mathbb E_{\ms\sim d^\pi_\mM}[R(\ms)] + \gamma\mathbb E_{\ms'\sim\mathbb P(\ms')}[V(\ms')]                                                     \\
	\end{aligned}
\end{equation*}

\begin{lemma}[CE-bounded value estimation gap]\label{ce-bound}
	\textit{
		Consider a pair of jointly distributed variables $(X,Y)\sim P$ over measurable spaces $(\mX,\Sigma)$
		and $(\mathcal Y,\Sigma)$ respectively. Let $Q(Y|X)$ be a distribution calibrated with respect to P. with $\ell_1$-calibration error.
		Then for an arbitrary function $g: \mathcal Y \to \R$ with which we are to take an expectation, the following inequality holds:
		\begin{equation}
			\left|\E_{y\sim \mathbb P(Y)}[g(y)]-\mathbb E_{x\sim P(X),\ y\sim
					Q(Y|X=x)}[g(y)]\right|\le
			\loce(Q)\max_y(g(y))
		\end{equation}
	}
\end{lemma}
\begin{proof}
	\begin{displaymath}
		\begin{aligned}
			 & |\mathbb E_{y\sim \mathbb P(Y)}[g(y)]-\mathbb E_{x\sim \mathbb P(X),\ y\sim
					Q(Y|X=x)}|                                                                                                   \\
			 & = \left|\sum_{y\in Y} g(y)P(Y=y)-\sum_{x\in X}P(X=x)\sum_{y\in Y} g(y)Q(Y=y|X=x)\right|                   \\
			 & = \left|\sum_{y\in Y} g(y)\left(P(Y=y)-\sum_{x\in X}P(X=x)Q(Y=y|X=x)\right)\right|                        \\
			 & = |\sum_{y\in Y}g(y)(\int_0^1dp\ P(Y=y, Q(Y=y|X=x)=p) P(X=x)                                              \\
			 & - \int_0^1dp \ p\cdot \sum_{x\in X}\mathbb I(Q(Y=y|X=x)=p) \cdot P(X=x))|                                 \\
			 & = |\sum_{y \in Y}g(y)\sum_{x \in
				X}P(X=x)\int^1_0dp(P(Y=y|Q(Y=y|X=x)=p)                                                                       \\
			 & -p\cdot \mathbb I(Q(Y=y|X=x)=p) )|                                                                        \\
			 & = \left|\sum_{y \in Y}g(y)\int^1_0dp\left(P(Y=y|Q(Y=y|X)=p)-p\cdot \mathbb I(Q(Y=y|X)=p) \right)\right|   \\
			 & = \left|\sum_{y \in Y}g(y)\int^1_0dp\left(P(Y=y|Q(Y=y|X)=p)-Q(Y=y|X)\right)\right|\                       \\
			 & = \left|\left\langle\,g(y),\int^1_0dp\left(P(Y=y|Q(Y=y|X)=p)-Q(Y=y|X)\right)\right\rangle_{Y}\right|      \\
			 & \leq \|g(y)\|_\infty \mathbb E_{y\in Y}\left\|\int^1_0dp\left(P(Y=y|Q(Y=y|X)=p)-Q(Y=y|X)\right)\right\|_1 \\
			 & \leq \|g(y)\|_\infty \mathbb E_{y\in Y}\int^1_0dp\left\|P(Y=y|Q(Y=y|X)=p)-Q(Y=y|X)\right\|_1              \\
			 & = \ell_1\text{-CE}(f)\max_{y \in Y}g(y)
		\end{aligned}
	\end{displaymath}
	where the sixth equality is because $p\cdot \mathbb I(Q(Y=y|X)=p) = Q(Y=y|X)=p$,
	the first inequality is because of Holder's inequality, the second inequality is because of Minkowski's inequality,
	and the last inequality follows from the definition of calibration error and $L_p$ space.

\end{proof}

\begin{theorem}
	\textit{
		Let $(S,A,T,R)$ be a discrete MDP and let $\pi$ be a policy over this MDP.
		Define $P$ be a joint distribution of state-actions and future states as
		outlined in Equation \ref{eq:cal1}. Then the difference of $\mJ$ the
		value of policy $\pi$ under the true dynamics T and $\hat {\mJ}$ the value
		under any virtual dynamics $\hat T$ is bounded by the latter's calibration error,
		\begin{equation}
			|\hat{\mathcal J}(\pi) - \mathcal J(\pi)| \leq \frac {\gamma R_{\mathrm{max}}}{1-\gamma}\ell_1\text{-CE}(\hat T)
		\end{equation}
	}
\end{theorem}
\begin{proof}
	\begin{displaymath}
		\begin{aligned}                & |\hat{\mathcal J}(\pi) - \mathcal J(\pi)| \\
			 & =\left|
			\left(\mathbb E_{\ms\sim \rho^\pi}R(\ms) + \gamma \mathbb E_{\ma \sim
					\pi(\cdot | \ms),\ms'\sim\hat T(\cdot | \ms,\ma)}[V(\ms')]\right)-\left(\mathbb
			E_{s\sim \rho^\pi}R(\ms) + \gamma \mathbb E_{\ms' \sim
					P(\ms’)}[V(\ms')]\right)\right|              \\& = \gamma\left|\mathbb E_{\ma \sim \pi(\cdot |
					\ms),\ms'\sim\hat T(\cdot | \ms,\ma)}[V(\ms')]-\mathbb E_{\ms' \sim
					P(\ms’)}[V(\ms')]\right|                     \\& \le \gamma \ell_1\text{-CE}(\hat T)\max_{\ms \in S}
			V(\ms)                                       \\ & \le \frac{\gamma R_{\mathrm{max}}}{1-\gamma}\ell_1\text{-CE}(\hat T),
		\end{aligned}
	\end{displaymath}
	where the first inequality directly follows from Lemma \ref{ce-bound}, and the second inequality follows from Assumption \ref{bound-q}.
\end{proof}

\begin{remark}
	The empirical policy evaluation function is pointwise equal to the ground-truth if the dynamic model is perfectly calibrated, as shown in \cite{malikCalibratedModelBasedDeep2019}.
\end{remark}

\begin{remark}
	However, \citeauthor{malikCalibratedModelBasedDeep2019} used Platt scaling for calibration, which induces unbounded calibration error \cite{kumarVerifiedUncertaintyCalibration2020}. One way to alleviate this is to use the scaling-binning calibration. By doing so, we can reduce the sampling error, so that DROMO can optimize over a lower bound with a smaller $\beta$.
\end{remark}

\section{Additional notes on automatic temperature control}
\label{auto-temp}
We also make extra notes on the temperature adjustment of the penalty term of a few
model-based algorithms.

\subsection{MOPO with automatically adjusted temperature}
Under the setup of MOPO (\cite{yuMOPOModelbasedOffline2020}), we change two things about the peanlty $\alpha u(\ms,\ma)$:
\begin{itemize}
	\item we replace the maximum variance with entropy, which is also dependent on the variance, i.e.
	      \[ u(\ms,\ma) := \frac 1 2 (1+ \log(2\pi \sigma^2))\\\ge \mH(\hat T(\ms,\ma))\]
	\item the $\alpha$ is updated according to:
	      \[\alpha \leftarrow \argmin_\alpha \fJ(\alpha),\]
	      where \(\mathfrak J(\alpha) := \E_{(\ms,\ma)\sim \rho^\pi_T}[\alpha \log \hat T(\ms'|\ms,\ma)+\alpha \delta]\).
\end{itemize}

The justification of our second change is inspired by SAC with dynamic temperature control (\cite{haarnojaSoftActorCriticAlgorithms2019}), and given below.

Formally, we want to solve the constrained optimization problem
\begin{equation}
	\max_\pi \E_{\rho^\pi} \left[\sum^T_{t=0} r(\ms_t,\ma_t)\right] \text{s.t.} \forall t, \E_{(\ms'_t|\ms_t,\ma_t)}[-\log \hat T(\ms_t,\ma_t)] \le \delta_T
\end{equation}
where $\delta_T$ is a maximum allowable expected entropy.
Note that, conservative policies tend to stay within the
support of the dataset, and the entropy term tends to vanish,
so we do not need to impose a lower bound on it.

Furthermore, we decompose the expected return
$\E_{\rho^\pi} \left[\sum^T_{t=0} r(\ms_t,\ma_t)\right]$
into a sum of expected rewards, using a dynamic programming approach in (??). Notice that policy $\pi_t$ has no effect on the policy at previous horizons, we can iteratively maximize the objective
\begin{equation}
	\max_{\pi_0} \left(\E[r(\ms_0,\ma_0)] + \max_{\pi_1} \left(\E[\cdots]+\max_{\pi_T} \E[r(\ms_T,\ma_T)]\right)\right)
\end{equation}
subject to that $\forall t,\ \E_{(\ms'_t|\ms_t,\ma_t)}[-\log T(\ms_t,\ma_t)] \le \delta_T$, and with loss of generality we consider the case where $\gamma =1$.

We begin by considering objective at the last timestep $T$:
\begin{displaymath}
	\text{maximize } \E_{(\ms_T,\ma_T)\sim \rho^\pi}[r(\ms_T,\ma_T)]
	\text{s.t. } \mH(\hat T(\ms_T,\ma_T))-\delta_T \le 0
\end{displaymath}
Define the following functions:
\begin{displaymath}
	\begin{aligned}
		h(\pi_T;\hat T) & = \mH(\hat T(\ms_T,\ma_T))-\delta_T
		:= \E_{\ms_T,\ma_T \sim \rho^{\pi_T}} [-\log \hat T(\ms_T,\ma_T)] - \delta_T \\
		f(\pi_T;\hat T) & = \begin{cases}
			\mathbb{E}_{(\ms_T,\ma_T) \sim \rho_{\pi}} [ r(\ms_T,\ma_T) ], & \text{if }h(\pi_T) \leq 0 \\
			-\infty,                                                       & \text{otherwise}
		\end{cases}
	\end{aligned}
\end{displaymath}
where, with an abuse of notation, we treat $h(\pi_T)$ and $h(\pi_T,\hat T)$ as interchangable. This changes the optimization problem into
\begin{equation}
	\text{maximize} f(\pi) \text{s.t. } h(\pi_T) \le 0
\end{equation}
To solve the problem with inequality constraint, we construct a lagrangian expression with a Lagrange expression with dual variable $\alpha_T$:

\begin{equation}
	L(\pi_T,\alpha_T) = f(\pi_T) - \alpha_T h(\pi_T)
\end{equation}
To recover $f(\cdot)$ we try to minimize $L(\pi_T,\alpha_T)$ w.r.t. $\alpha_T$ under a fixed $\pi_T$
\begin{itemize}
	\item when $h(\pi_T) \le 0$, the best we can do is setting $\alpha_T = 0$
	\item when $h(\pi_T) > 0$, we have $L(\pi_T,\alpha_T) \to -\infty$ as $\alpha_T \to \infty$.
\end{itemize}
Either way we have $L(\pi_T, \cdot) = -\infty = f(\pi_T)$, thus we can recover
\begin{equation}
	f(\pi_T) = \min_{\alpha_T \ge 0} L(\pi_T,\alpha_T)
\end{equation}
Meanwhile, we want to maximize the above in the sense that
\begin{equation}
	\max_{\pi_T}f(\pi_T) = \min_{\alpha_T \ge 0} \max_{\pi_T} L(\pi_T,\alpha_T)
\end{equation}
It follows that,
\begin{displaymath}
	\begin{aligned}
		\max_{\pi_T} \E[r(\ms_T,\ma_T)] & = \max_{\pi_T} f(\pi_T)                                                                                                                \\
		                                & = \min_{\alpha_T \ge 0} \max_{\pi_T} f(\pi_T) - \alpha_T h(\pi_T)                                                                      \\
		                                & = \min_{\alpha_T \ge 0} \max_{\pi_T} f(\pi_T) - \alpha_T h(\pi_T)                                                                      \\
		                                & = \min_{\alpha_T \ge 0} \max_{\pi_T} \E_{\rho^{\pi_T}}[r(\ms_T,\ma_T)- \alpha_T (\E_{\rho^{\pi_T}}[ (-\log \hat T(\ms,\ma)]- \delta_T) \\
		                                & = \min_{\alpha_T \ge 0} \max_{\pi_T} \E_{\rho^{\pi_T}}[r(\ms_T,\ma_T)-\alpha_T \mH(\hat T(\ms,\ma)) + \alpha_T \delta_T]
	\end{aligned}
\end{displaymath}
We can compute the optimal $\pi_T$ and $\alpha_T$ iteratively.
First we fix the current $\alpha_T$,
and search for the $\pi^*_T$ maximizes $L(\pi^*_T,\alpha_T)$.
Then fix $\pi^*_T$ and $\alpha^*_T$ that minimizes $L(\pi^*_T,\alpha^*_T)$.
\begin{equation}
	\begin{aligned}
		\label{eq:alpha-goal1}
		\pi^*_T    & = \argmax_{\pi_T}\E_{\rho^{\pi_T}}[r(\ms_T,\ma_T)-\alpha_T \mH(\hat T(\ms,\ma)) + \alpha_T \delta_T] \\
		\alpha^*_T & = \argmin_{\alpha_T} \E_{\rho^{\pi_T}}[-\alpha_T \mH(\hat T(\ms,\ma)) + \alpha_T \delta_T]
	\end{aligned}
\end{equation}

Thus, $\max_{\pi_T} \E[r(\ms_T,\ma_T)] =
	\E_{\ms_T,\ma_T \sim \rho^{\pi^*_T}}[r(\ms_T,\ma_T)-\alpha^*_T \mH(\hat T(\ms_T,\ma_T)) + \alpha^*_T \delta_T]$
We now turn to the second optimization:

\begin{displaymath}
	\begin{aligned}
		Q_{T-1}(\ms_{T-1},\ma_{T-1})   & = r(\ms_{T-1},\ma_{T-1})+\E[Q(\ms_T+\ma_T)+ \alpha_T \log(\hat T(\ms_T,\ma_T))] \\
		                               & = r(\ms_{T-1},\ma_{T-1})+\E[r(\ms_T+\ma_T)- \alpha_T \mH(\hat T(\ms_T,\ma_T))]  \\
		Q^*_{T-1}(\ms_{T-1},\ma_{T-1}) & = r(\ms_{T-1},\ma_{T-1})+\max_{\pi_T}
		\E[r(\ms_T+\ma_T)- \alpha_T \mH(\hat T(\ms_T,\ma_T))]                                                            \\
	\end{aligned}
\end{displaymath}
It follows that,
\begin{displaymath}
	\begin{aligned}
		 & \max_{\pi_{T-1}} \left(\E[r(\ms_{T-1},\ma_{T-1})]+\max_{\pi_T} \E[r(\ms_T,\ma_T)]\right)                                                                                                     \\
		 & = \max_{\pi_{T-1}} \E \left(Q^*_{T-1}(\ms_{T-1},\ma_{T-1}) + \alpha^*_T \mH(\hat T(\ms_T,\ma_T))\right)                                                                                      \\
		 & = \min_{\alpha_{T-1}} \max_{\pi_{T-1}} \E \left(Q^*_{T-1}(\ms_{T-1},\ma_{T-1}) + \alpha^*_T \mH(\hat T(\ms_T,\ma_T)) - \alpha_{T-1}(\mH(\hat T(\ms_{T-1},\ma_{T-1}))-\delta_T)\right)        \\
		 & = \min_{\alpha_{T-1}} \max_{\pi_{T-1}} \E \left(Q^*_{T-1}(\ms_{T-1},\ma_{T-1}) - \alpha_{T-1}\mH(\hat T(\ms_{T-1},\ma_{T-1}))+\alpha_{T-1}\delta_T\right)+\alpha^*_T \E \mH(\hat T(\ms,\ma)) \\
	\end{aligned}
\end{displaymath}
Similar to the previous step,
\begin{equation}
	\begin{aligned}
		\pi^*_T    & = \argmax_{\pi_{T-1}} \E_{\rho^{\pi_{T-1}}} \left(Q^*_{T-1}(\ms_{T-1},\ma_{T-1}) - \alpha_{T-1}\mH(\hat T(\ms_{T-1},\ma_{T-1}))+\alpha_{T-1}\delta_T\right) \\
		\alpha^*_T & = \argmin_{\alpha_{T-1}}
		\E_{\rho^{\pi_{T-1}}}[-\alpha_{T-1} \mH(\hat T(\ms,\ma)) + \alpha_{T-1} \delta_T]
	\end{aligned}
	\label{eq:alpha-goal2}
\end{equation}
Notice the similarity of the objective of $\alpha_t$ in Equation \ref{eq:alpha-goal1} and \ref{eq:alpha-goal2}. As the
process repeat, we can search for the optimal temperature scaling in every step by minimizing the
following objective
\begin{equation}
	\fJ(\alpha) := \E[\alpha \log{\hat T(\ms,\ma)}+\alpha \delta_T].
	\label{eq:temp-update-mopo}
\end{equation}
The final MOPO instantiation is as follows


\section{Missing Proofs in Section \ref{sec:dr-combo}}
\subsection{Missing Proofs in Section \ref{sec:dromo-lb}}
\label{theorem:proofs}
In the section, we justify our theoretical results stated in Section \ref{sec:dr-combo}.
Note that all the statements hold under finite state and action space (i.e., $|\mS|,|\mA| < \infty$).
\begin{proof}[Proof of Lemma \ref{q-update-lemma}]
	First we set the derivative of Equation \ref{eq:dr-combo} w.r.t. a single state-action pair to $0$.
	\begin{displaymath}
		\begin{aligned}
			0 & \leftarrow d_f(s,a)Q(s,a)\dsq-d_f(s,a)(\hat \mB^\pi \hat Q^k)(s,a)\dsq                                                         \\
			  & + \alpha \frac {d_f(s)(\pi^f(a|s)Q(s,a)\dsq-\pi^f(s)Q(s,a)\pi^f(s)\dsq)}{\sqrt{\var_{\pi^f}(Q(s,\cdot))|\mD|d^{\pi^\beta}(s)}} \\
			  & + \beta(\rho(s,a)-d(s,a))\dsq                                                                                                  \\
			  & = d_f(s,a)Q(s,a)-d_f(s,a)(\hat \mB^\pi \hat Q^k)(s,a)\dsq                                                                      \\
			  & + \alpha \frac {d_f(s,a)(1-\pi^f(a|s))Q(s,a)\dsq}{\sqrt{\var_{\pi^f}(Q(s,\cdot))|\mD(s)|}}                                     \\
			  & + \beta(\rho(s,a)-d(s,a))\dsq                                                                                                  \\
		\end{aligned}
	\end{displaymath}
	where the second equality is a result of $|\mD(s)| \approx |\mD|d^{\pi^\beta}(s)|$.

	It follows that,
	\begin{displaymath}
		\begin{aligned}
			 & \left(1+ \frac {\alpha(1-\pi^f(a|s))}{\COne}\right)d_f(s,a)Q(s,a)                                                                                            \\
			 & = d_f(s,a)(\hat \mB^\pi \hat Q^k)(s,a) - \beta(\rho(s,a)-d(s,a))                                                                                             \\
			 & Q(s,a) = \left(1- \frac {\alpha (1-\pi^f(a|s))}{\alpha(1-\pi^f(a|s))+\COne}\right)\left((\hat \mB^\pi \hat Q^k)(s,a) - \beta \frac{\rho-d}{d_f}(s,a)\right). \\
		\end{aligned}
	\end{displaymath}
\end{proof}

Next, we lay down some assumptions for the proof of Theorem \ref{q-update}.
\begin{assumption}[The variance is lower bounded]
	\label{bound-var}
	\textit{Given $f \in [0,1]$, any $\psi \in \Psi$ and any $s \in S$, $\var_{\pi^f}Q^\psi(s,\cdot) \ge \kappa_{\var}(\pi)^{-1}$.}
\end{assumption}
\begin{assumption}[Boundedness of reward function]
	\label{bound-q}
	\textit{for all $s \in \mS, a \in \mA$, $|r(s,a)| \le R_{\max}$.}
\end{assumption}

As a result of Assumption , given a dynamic model $T$ and policy $\pi$. For simplicity, define
\begin{displaymath}
	G(s_t,a_t) := \sum^\infty_{k=0} \gamma^k r(s_{t+k+1},a_{t+k+1})
\end{displaymath}
Then the Q-function and value function satisfies:
\begin{displaymath}
	\begin{aligned}
		Q(s,a)                   & = \E_{s_{t+k+1},a_{t+k+1} \sim \rho^\pi_T;\ k\ge 0}[G(s_t,a_t)|s_t = s, a_t = a] \\
		-	(1-\gamma)^{-1}R_{\max} & \le Q(s,a) \le (1-\gamma)^{-1}R_{\max},\ \forall s \in \mS, a \in \mA
	\end{aligned}
\end{displaymath}
and
\begin{displaymath}
	\begin{aligned}
		V(s)                     & = \E_{s_{t+k+1},a_{t+k+1} \sim \rho^\pi_T;\ k\ge 0}[G(s_t,a_t)|s_t = s] \\
		-	(1-\gamma)^{-1}R_{\max} & \le V(s) \le (1-\gamma)^{-1}R_{\max},\ \forall s \in \mS
	\end{aligned}
\end{displaymath}

\begin{proof}[Proof of Theorem \ref{asymp}]

	We know that
	$(\hat \mB^k Q)(s,a)
		= f(\hat \mB^k_{\widetilde \mM} Q)(s,a)
		+ (1-f)(\hat \mB^k_{\widehat \mM} Q)(s,a)$.
	Based on a result directly from \cite{yuCOMBOConservativeOffline2021}, the right part of RHS of Equation \ref{q-update} satisfies:
	\begin{equation}
		\label{combo-result}
		\begin{aligned}
			\forall s \in \mS, a \in \mA,\  & (\hat \mB^\pi \hat Q^k)(s,a) - \beta \frac{\rho-d}{d_f}(s,a) \\
			                                & \le (\mB^\pi \hat Q^k)(s,a)
			- \beta \frac {\rho-d}{d_f}
			+ (1-f)\left[|r-r_{\widehat \mM}|
			+\frac {2\gamma R_{\max}}{1-\gamma}\TV(T,\hat T)\right]                                        \\
			                                & + \frac{fC_{r,T,\delta}R_{\max}}{(1-\gamma)\sqrt{|\mD|}}
		\end{aligned}
	\end{equation}
	Because (a) Equation \ref{combo-result} upper bounds the Q-function pointwise,
	(b) the fixed point of the Bellman backup of the LHS will be pointwise smaller than the fixed point of RHS.
	Hence, plugging Equation \ref{combo-result} into Equation \ref{q-update-lemma} gives
	\begin{equation}
		\begin{aligned}
			\hat Q^\pi(\ms,\ma) \le & \underbrace{\left(1-\frac
				{\lambda_{\alpha,\pi,f}(s,a)}{\lambda_{\alpha,\pi,f}(s,a)+\sqrt{|\mD(\ms)|\var_{\pi^f}(\hat
						Q^k)}}\right)}_{\le 1}\Bigg\{S^\pi r_{\mM} - \beta \left[\frac 1 c_s S^\pi \left[\frac
					{\rho-d}{d_f}\right]\right](s,a)                                          \\
			                        & + f\left[ S^\pi\left[|R-R_{\widehat \mM}|+\frac
			{2\gamma R_{\max}}{1-\gamma}\TV(P,P_{\widehat \mM}) \right]
			\right]                                                                   \\
			                        & + (1-f)\left[S^\pi \left[\frac
			{C_{r,T,\delta}R_{\max}}{(1-\gamma)\sqrt{|D|}} \right]
			\right](s,a)\Bigg\}
		\end{aligned}
	\end{equation}
\end{proof}

\begin{corollary}[Restatement of Corollary \ref{corollary-smaller}]
	\label{corollary-smaller-bigger}

	\textit{For $\beta \ge c_{\rho,f}$ and an initial state distribution $\mu(s)$,
		we have $\E_{s\sim \mu(s), a\sim \pi(a|s)}[\hat Q^\pi(s,a)] = \E_{s\sim \mu(s), a\sim \pi(a|s)}[Q^\pi(s,a)]$,
		where $c_{\rho,f}$ is given by:
		\begin{displaymath}
			\begin{aligned}
				c_{\rho,f} = & \nu(\rho,f)^{-1}\Biggl\{|R-R_{\widehat \mM}|+\frac{2\gamma R_{\max}}{1-\gamma}\TV(P,P_{\widehat \mM}) \\&+ \frac {C_{r,T,\delta}R_{\max}}{(1-\gamma)\sqrt{|D|}}+\frac{\alpha(1-\|\pi\|^2_2+\TV(\pi,\pi^\beta))R_{\max} \sqrt{\kappa_\var(\pi)|\mS|(\CQL(\rho,d^{\pi^\beta})+1)}}{(1-\gamma)\sqrt{|\mD|}}\Biggl\}
			\end{aligned}
		\end{displaymath}
	}
\end{corollary}
\begin{proof}
	From Lemma \ref{q-update-lemma} we know that:
	\begin{equation}
		\label{eq:ass2-result}
		\begin{aligned}
			 & \left(1+ \frac {\lambda_{\alpha,\pi,f}(s,a)}{\COne}\right)\hat Q^{k+1}(s,a)                                                              \\
			 & =(\hat \mB^\pi \hat Q^k)(s,a) - \beta\frac{\rho(s,a)-d(s,a)}{d_f(s,a)}                                                                   \\
			 & = (\mB^\pi_{\widehat \mM} \hat Q^k)(s,a) - \beta \frac{\rho-d}{d_f}(s,a)+f(\mB^\pi_{\widetilde \mM}-\mB^\pi_{\widehat \mM})\hat Q^k(s,a)
		\end{aligned}
	\end{equation}
	where by Assumption \ref{ass:bound-q}
	the third term is bounded by
	\begin{displaymath}
		\begin{aligned}
			\left|\left((\mB^\pi_{\widetilde \mM}-\mB^\pi_{\widehat \mM})\hat Q^k(s,a)\right)\right| & \le \underbrace{\left[|R-R_{\widehat \mM}|+\frac
			{2\gamma R_{\max}}{1-\gamma}\TV(P,P_{\widehat \mM}) \right]+ \left[\frac {C_{r,T,\delta}R_{\max}}{(1-\gamma)\sqrt{|D|}} \right]}_{:=\Delta_{r,T,\delta}}
		\end{aligned}
	\end{displaymath}
	Taking expectation over \ref{eq:ass2-result} with respect to $\mu(s)\pi(a|s)$ gives
	\begin{displaymath}
		\begin{aligned}
			 & \E_{s\sim\mu(s),a\sim\pi(a|s)}\left[\left(1+ \frac {\lambda_{\alpha,\pi,f}(s,a)}{\COne}\right)\hat Q^\pi(s,a)\right]                                                                                      \\
			 & \le \E_{s\sim\mu(s),a\sim\pi(a|s)}[Q^\pi(s,a)]-\beta \E_{s,a\sim \rho}\left[\frac{\rho-d}{d_f}\right]+f \Delta_{r,T,\delta}                                                                               \\
			 & \E_{s\sim\mu(s),a\sim\pi(a|s)}[\hat Q^\pi(s,a)] \le \E_{s\sim\mu(s),a\sim\pi(a|s)}[Q^\pi(s,a)]                                                                                                            \\
			 & - \beta \E_{s,a\sim \rho}\left[\frac{\rho-d}{d_f}\right]+f \Delta_{r,T,\delta}                                                                                                                            \\
			 & -\E_\mu\left[\frac{\lambda_{\alpha,\pi,f}(s,a)}{\lambda_{\alpha,\pi,f}(s,a)+\COne}\right]\left\{\E_{\mu,\pi}[Q(s,a)]-\beta \E_{s,a\sim \rho}\left[\frac{\rho-d}{d_f}\right]+f \Delta_{r,T,\delta}\right\} \\
			 & =\E_{s\sim\mu(s),a\sim\pi(a|s)}[Q^\pi(s,a)] - \beta \E_{s,a\sim \rho}\left[\frac{\rho-d}{d_f}\right]+f \Delta_{r,T,\delta}                                                                                \\
			 & -\E_\mu\left[\frac{\lambda_{\alpha,\pi,f}(s,a)}{\lambda_{\alpha,\pi,f}(s,a)+\COne}\right]\left\{\E_{\mu,\pi}[Q(s,a)]-\beta \E_{s,a\sim \rho}\left[\frac{\rho-d}{d_f}\right]+f \Delta_{r,T,\delta}\right\} \\
		\end{aligned}
	\end{displaymath}
	Let $\nu(\rho,f)$ for simplicity. In order to guarantee underestimation, we need to have $\beta$ such that,
	\begin{displaymath}
		\begin{aligned}
			 & \beta \E_\mu\left[1-\frac{\lambda_{\alpha,\pi,f}(s,a)}{\lambda_{\alpha,\pi,f}(s,a)+\COne}\right]\nu(\rho,f)                       \\
			 & \ge \E_\mu\left[1-\frac{\lambda_{\alpha,\pi,f}(s,a)}{\lambda_{\alpha,\pi,f}(s,a)+\COne}\right]\Delta_{r,T,\delta}                 \\
			 & - \E_{\mu,\pi}\left[\frac{\lambda_{\alpha,\pi,f}(s,a)}{\lambda_{\alpha,\pi,f}(s,a)+\COne} Q^\pi(s,a)\right]                       \\
			 & \beta \ge \nu(\rho,f)^{-1}\left\{\Delta_{,r,T,\delta} - \frac{\E_{\mu,\pi}[\lambda_{\alpha,\pi,f}(s,a)Q^\pi(s,a)]}{\COne}\right\}
		\end{aligned}
	\end{displaymath}

	Due to Assumption \ref{bound-q} we can lower bound the Q-estimate on the RHS with $-(1-\gamma)^-1R_{\max}$.
	In addition, $\beta$ is, by definition, greater than $0$, so we also restrict $\alpha$ to satisfy this.
	\begin{displaymath}
		\begin{aligned}
			\Delta_{r,T,\delta} \ge \frac{\alpha(|\mA|-1)\E_{\mu,\pi}[Q^\pi(s,a)]}{\COne} \\
			\alpha \le \frac{(1-\gamma)\COne}{(|\mA|-1)R_{\max}}
		\end{aligned}
	\end{displaymath}
	where the last inequality is because Assumption \ref{bound-q} also upper bounds the $Q$-function, and combining the result of Lemma \ref{bound-c}.
\end{proof}
Next, we prove that DROMO has a gap-expanding property that is also enjoyed by CQL \cite{kumarConservativeQLearningOffline2020}.
\begin{proof}[Proof of Theorem \ref{thrm:gap-expanding}]

	Recall the Q-function update with respect to a single $(s,a)$-pair
	as in Equation \ref{eq:ass2-result}. For any marginal distribution $\rho(s,a)$,

	denote $ C_{\var}(\rho) := \E_{s,a\sim\rho(s,a)}\left[{\lambda_{\alpha,\pi,f}(s,a)}/({\lambda_{\alpha,\pi,f}(s,a)+\COne})\right]$ for simplicity.

	Taking expectation over $\rho(s,a)=d^\pi_{\widehat \mM}\pi(a|s)$, we have that
	\begin{displaymath}
		\begin{aligned}
			\E_{s,a\sim\rho}[\hat Q^{k+1}(s,a)] & = \left(1 -C_{\var}(\rho)\right) \left(\E_{s,a\sim\rho}[\hat \mB^\pi \hat Q^k]-\beta \underbrace{\E_{s,a\sim\rho}\left[\frac{\rho-d}{d_f}\right]}_{:= \chi_1 \ge 0,\ \text{Theorem 2 of \cite{yuCOMBOConservativeOffline2021}}}\right) \\
			                                    & = \left(1 -C_{\var}(\rho)\right) \left(\E_{s,a\sim\rho}[\mB^\pi \hat Q^k]-\beta \chi_1 + \E_{s,a\sim\rho}[\Delta(s,a)]\right)
		\end{aligned}
	\end{displaymath}
	where $\Delta(\cdot,\cdot)$ is given by $\Delta(s,a) := (\hat \mB^k \hat Q)(s,a)
		= f(\hat \mB^\pi_{\widetilde \mM} \hat Q - \mB^\pi \hat Q)(s,a)
		+ (1-f)(\hat \mB^k_{\widehat \mM} \hat Q - \mB^\pi \hat Q)(s,a)$.

	Similarly, note that if we take expectation over the marginal distribution
	induced by the dataset, the penalty term is negative:
	\begin{displaymath}
		\begin{aligned}
			\E_{s,a\sim d}[\hat Q^{k+1}(s,a)] & = \left(1 -C_{\var}(d)\right) \left(\E_{s,a\sim d}[\hat \mB^\pi \hat Q^k]-\beta \underbrace{\E_{s,a\sim d}\left[\frac{\rho-d}{d_f}\right]}_{:= -\chi_2 \le 0,\ \text{Corollary 8 of \cite{yuCOMBOConservativeOffline2021}}}\right) \\
			                                  & = \left(1 -C_{\var}(d)\right) \left(\E_{s,a\sim d}[\mB^\pi \hat Q^k]+\beta \chi_2 + \E_{s,a\sim d}[\Delta(s,a)]\right)
		\end{aligned}
	\end{displaymath}

	It follows that
	\begin{displaymath}
		\begin{aligned}
			 & \E_{s,a\sim d}[\hat Q^{k+1}(s,a)] - \E_{s,a\sim \rho}[\hat Q^{k+1}(s,a)]               \\
			 & = (d-\rho)^T(\mB^\pi \hat Q^k) +\beta(\chi_1+\chi_2) + (d-\rho)^T\Delta(s,a)           \\
			 & - C_{\var}(d)\left(d^T(\mB^\pi \hat Q^k) +\beta\chi_2 + d^T\Delta(s,a)\right)          \\
			 & + C_{\var}(\rho)\left(\rho^T(\mB^\pi \hat Q^k) +\beta\chi_1 + \rho^T\Delta(s,a)\right) \\
		\end{aligned}
	\end{displaymath}
	Adding $\left(-\E_{s,a\sim d}[Q^{k+1}(s,a)]+ \E_{s,a\sim\rho}[Q^{k+1}(s,a)]\right)$ to both sides we get
	\begin{displaymath}
		\begin{aligned}
			  & \E_{s,a\sim d}[\hat Q^{k+1}(s,a)] - \E_{s,a\sim \rho}[\hat Q^{k+1}(s,a)]-\E_{s,a\sim d}[Q^{k+1}(s,a)]+ \E_{s,a\sim\rho}[Q^{k+1}(s,a)]                               \\
			= & (d-\rho)^T(\mB^\pi \hat Q^k) +\beta(\chi_1+\chi_2) + (d-\rho)^T\Delta(s,a) - (d-\rho)^T \mB^\pi Q^k                                                                 \\
			  & - C_{\var}(d)\left(d^T(\mB^\pi \hat Q^k) +\beta\chi_2 + d^T\Delta(s,a)\right)+ C_{\var}(\rho)\left(\rho^T(\mB^\pi \hat Q^k) +\beta\chi_1 + \rho^T\Delta(s,a)\right) \\
			= & \beta((1+C_\var(\rho))\chi_1+(1-C_{\var}(d))\chi_2) + \left((1-C_{\var}(d))d-(1-C_\var(\rho))\rho\right)^T \Delta(s,a)                                              \\
			  & + (d-\rho)^T\left(\mB^\pi \hat Q^k-\mB^\pi Q^k \right)+ (C_{\var}(\rho)\rho-C_{\var}(d)d)^T(\mB^\pi \hat Q^k)
		\end{aligned}
	\end{displaymath}
	In order that LHS $\ge 0$, we require $\beta$ to satisfy,
	\begin{displaymath}
		\begin{aligned}
			\beta \ge & \frac{(d-\rho)^T\left(\mB^\pi (\hat Q^k-Q^k) \right)+ (C_{\var}(\rho)\rho-C_{\var}(d)d)^T(\mB^\pi \hat Q^k)}{(1+C_\var(\rho))\chi_1+(1-C_{\var}(d))\chi_2} \\
			          & + \frac{\left((1-C_{\var}(d))d-(1-C_\var(\rho))\rho\right)^T \Delta(s,a)}{(1+C_\var(\rho))\chi_1+(1-C_{\var}(d))\chi_2}                                    \\
		\end{aligned}
	\end{displaymath}
	thus giving the desired result.

	Because of Assumption \ref{ass:bound-q}, we can bound $\Delta(\cdot,\cdot)$ with $(1-f)\left[|r-r_{\widehat \mM}|
		+\frac {2\gamma R_{\max}}{1-\gamma}\TV(T,\hat T)\right] + \frac{fC_{r,T,\delta}R_{\max}}{(1-\gamma)\sqrt{|\mD|}}$
\end{proof}

\begin{lemma}\label{bound-c}
	\textit{Let $\lambda_{\alpha,\pi,f}(s,a)$ be defined as in Lemma \ref{q-update-lemma}, given marginal distribution $\rho(s,a) =\rho(s)\pi(a|s)$, then
		\begin{displaymath}
			\begin{aligned}
				    & \E_{s,a\sim\rho}\left[\frac{\lambda_{\alpha,\pi,f}(s,a)}{\COne}\right]                                                          \\
				\le & \alpha \left(1-\|\pi\|^2_{L^2}+f\TV(\pi,\pi^\beta)\right)\sqrt{\frac{\kappa_\var(\pi)|\mS|(\CQL(\rho,d^{\pi^\beta})+1)}{|\mD|}}
			\end{aligned}
		\end{displaymath}
	}
\end{lemma}
\begin{proof}
	\begin{displaymath}
		\begin{aligned}
			 & = \alpha\sum_{s,a} \frac{\rho(s)\pi(a|s)(1-\pi^f(a|s))}{\COne}                                                                  \\
			 & = \frac{\alpha}{\sqrt{|\mD|}}\sum_s\frac{\rho(s)}{\sqrt{\var_{\pi^f}(Q(s,\cdot))d^{\pi^\beta}(s)}}\sum_a \pi(a|s)(1-\pi^f(a|s))
		\end{aligned}
	\end{displaymath}
	Let $\kappa(s) := \frac{\rho(s)}{\sqrt{d^{\pi^\beta}(s)}}$. Then we can write:
	\begin{equation}
		\label{bound-d-lem}
		\begin{aligned}
			\CQL(\rho,d^{\pi^\beta})   & = \sum_s \frac{\rho(s)^2}{d^{\pi^\beta}(s)} - 1                           \\
			\CQL(\rho,d^{\pi^\beta})+1 & = \sum_s \kappa(s)^2                                                      \\
			\CQL(\rho,d^{\pi^\beta})+1 & \le \left(\sum_s \kappa(s)\right)^2 \le |\mS|(\CQL(\rho,d^{\pi^\beta})+1)
		\end{aligned}
	\end{equation}
	Similarly, the summanded-over-$a$ part can be bounded as well.
	\begin{equation}
		\label{bound-pi-lem}
		\begin{aligned}
			\sum_a \pi(a|s)(1-\pi^f(a|s)) & = \sum_a \pi(a|s)\left[1-\pi(a|s)+f(\pi(a|s)-\pi^\beta(a|s))\right] \\
			                              & \le f\TV(\pi,\pi^\beta)+\sum_a \pi(a|s)(1-\pi(a|s))                 \\
			                              & = 1-\|\pi(a|s)\|^2_{L^2}+f\TV(\pi,\pi^\beta)
		\end{aligned}
	\end{equation}
	Combining Equation \ref{bound-pi-lem}, Equation \ref{bound-d-lem} and Assumption \ref{bound-var} gives the claim.
\end{proof}

\subsection{Missing Proofs in Sections \ref{sec:dromo-lin} and \ref{sec:dromo-nonlin}}
\begin{proof}[Proof of Theorem \ref{linear-theo}]
	We first set the gradient to zero and substituting $Q(s,a)=\omega^T \mat F(s,a)$,
	\begin{displaymath}
		\begin{aligned}
			0 & = \sum_{s,a}d_f(s,a)\left(\hat Q(s,a)-\hat \mB^\pi \hat Q^k(s,a)\right)\mat F(s,a)                                                                                                \\
			  & + \beta^k \sum_{s,a}(\rho(s,a)-d(s,a)) \mat F(s,a)+ \alpha^k \sqrt{\frac1n}\frac{\E_{d_f}\left[(\mat F-\E_{d_f}[\mat{F}])(\mat F\omega^k-\E_{d_f}[\mat F\omega^k])\right]}{\CThr} \\
			  & = \sum_{s,a}d_f(s,a)\left(\hat Q(s,a)-\hat\mB^\pi \hat Q^k(s,a)\right)\mat F(s,a)                                                                                                 \\
			  & + \beta^k \sum_{s,a}(\rho(s,a)-d(s,a)) \mat F(s,a)+ \alpha^k \sqrt{\frac1n}\frac{\cov(\mat{F},\mat{F}\omega^k)}{\CThr}
		\end{aligned}
	\end{displaymath}

	By simple algebraic manipulation, and substituting $\mat U = diag(\rho(s,a))$,$\mat V = diag(d(s,a))$, and $\mat{D_f}=diag(d^f(s,a))$ we have that

	\begin{displaymath}
		\begin{aligned}
			\mat{F}^T\mat{D_f F}\omega^{k+1} & =\mat{F}^T\mat{D_f}\left(\mB^\pi \hat Q^k\right)-\beta^k\mat{F}^T \mat{(U-V)}-\alpha^k n^{-\frac 1 2}\frac{\cov_{d_f}(\mat{F,F\omega}^k)}{\CThr}+\mat{F}^T\mat{D_f}\Delta
		\end{aligned}
	\end{displaymath}
	In order to show similar property in Corollory \ref{corollary-smaller-bigger}, we take the expectation under $\mu(s)\cdot\pi(a|s)$, we have the following. Specifically, we need to reason with the terms $C_{var}$ and $(\star)$,
	\begin{equation}
		\begin{aligned}
			\E_{\mu(s),\pi(a|s)}[\hat Q^{k+1}(s,a)] & :=\left(\mu\cdot\pi\right)^T\mat{F}\omega^{k+1}                                                                                                                                                             \\
			                                        & = \underbrace{\left(\mu\cdot\pi\right)^T\mat{F}\left(\mat F^T \mat {D_fF}\right)^{-1}\mat F^T\mat{D_f}\left(\hat\mB^\pi \hat Q^k\right)}_{\text{LSTD-Q (\cite{lagoudakisLeastSquaresPolicyIteration2003})}} \\
			                                        & -\beta^k\underbrace{\left(\mu\cdot\pi\right)^T\mat{F}\left(\mat F^T \mat {D_fF}\right)^{-1}\mat F^T\left(\rho-d\right)}_{=:(\star)}                                                                         \\
			                                        & -\alpha^k \underbrace{n^{-\frac 1 2}\left(\mu\cdot\pi\right)^T\mat{F}\left(\mat F^T \mat {D_fF}\right)^{-1}\frac{\cov_{d_f}(\mat{F,F\omega}^k)}{\CThr}}_{=:C_\var\ge0}                                      \\
		\end{aligned}
	\end{equation}
	The $C_{var}$ term is greater than $0$ by definition. Also, we expect the stared part to be greater than zero, and indeed we can show
	\begin{displaymath}
		\begin{aligned}
			(\star) & = \beta^k\left(\mu\cdot\pi\right)^T\mat{F}\left(\mat F^T \mat {D_fF}\right)^{-1}\mat F^T\left(\rho-d\right) \\
			        & = \beta^k \left(\mu\cdot\pi\right)^T \mat{D_f}^{-1}(\rho-d)                                                 \\
			        & = \beta^k \sum_{s,a} \mu(s)\pi(a|s)\frac{\rho(s,a)-d(s,a)}{d_f(s,a)}
			\\
			        & = \beta^k \sum_{s,a}(1-\gamma P^\pi_{\widehat \mM})\rho(s,a)\frac{\rho(s,a)-d(s,a)}{d_f(s,a)}
		\end{aligned}
	\end{displaymath}
	where $P^\pi_{\widehat \mM} = \hat T \odot \pi$. The fourth inequality is because $\rho(s,a) = (\mu\cdot\pi)^T(1-\gamma P^\pi)^{-1}(s,a)$. Since $1-\gamma P^\pi_{\widehat \mM}>0$ and, by Lemma 1 of \cite{yuCOMBOConservativeOffline2021}, $\sum_{s,a}\rho(s,a)\frac{\rho(s,a)-d(s,a)}{d_f(s,a)}>0$, we can conclude that $(\star)>0$.
	Thus we can adjust
	\begin{displaymath}
		\begin{aligned}
			 & \E_{\mu(s),\pi(a|s)}[\hat Q^{k+1}(s,a)]                                                                                                                                            \\
			 & := \E_{\mu(s),\pi(a|s)}[\hat Q^{k+1}_{\lstd}(s,a)]                                                                                                                                 \\
			 & -\alpha^k C_{\var} - \beta^k (\star)                                                                                                                                               \\
			 & = \E_{\mu(s),\pi(a|s)}[Q^{k+1}(s,a)]                                                                                                                                               \\
			 & \underbrace{+\E_{\mu(s),\pi(a|s)}[\hat Q^{k+1}_{\lstd}(s,a)-Q^{k+1}(s,a)] + \E_{\mu(s),\pi(a|s)}[\Delta(s,a)]-\alpha^k C_{\var} - \beta^k (\star)}_{\text{make smaller than zero}} \\
		\end{aligned}
	\end{displaymath}
	This implies that $\hat \mJ^{k+1}(\pi) \le \mJ^{k+1}_{\text{LSTD}}(\pi)$ given the same policy $\pi$.
	In order that bracketed part is smaller than zero, we can let $\beta^k$ satisfy,
	\begin{displaymath}
		\begin{aligned}
			(\star)\beta^k & \ge \E_{\mu(s),\pi(a|s)}[\hat Q^{k+1}_{\lstd}(s,a)-Q^{k+1}(s,a)] + \E_{\mu(s),\pi(a|s)}[\Delta(s,a)]                                                                                                                                                                        \\
			\beta^k        & \ge \frac{\E_{\mu(s),\pi(a|s)}[\hat Q^{k+1}_{\lstd}(s,a)-Q^{k+1}(s,a)] + \E_{\mu(s),\pi(a|s)}[\Delta(s,a)]-C_\var}{(\star)}                                                                                                                                                 \\
			               & = \frac{(\mu\cdot\pi)^T\left[\mat{F}\left(\mat F^T\mat{D_fF}\right)^{-1}\mat{F}^T \left(\hat \mB^\pi \hat Q^k\right)-\left(\mB^\pi \hat Q^k\right)\right]-C_\var}{\left(\mu\cdot\pi\right)^T\mat{F}\left(\mat F^T \mat {D_fF}\right)^{-1}\mat F^T(\rho-d)}                  \\
			               & = \frac{(\mu\cdot\pi)^T\left[\mat{F}\left(\mat F^T\mat{D_fF}\right)^{-1}\mat{F}^T \left(\hat \mB^\pi \hat Q^k\right)-\left(\hat \mB^\pi \hat Q^k\right)+\Delta(s,a)\right]-C_\var}{\left(\mu\cdot\pi\right)^T\mat{F}\left(\mat F^T \mat {D_fF}\right)^{-1}\mat F^T(\rho-d)} \\
		\end{aligned}
	\end{displaymath}
	Taking the upper bound of $\Delta(\cdot,\cdot)=f(\hat \mB^\pi_{\widetilde \mM} \hat Q - \mB^\pi \hat Q)(s,a)
		+ (1-f)(\hat \mB^k_{\widehat \mM} \hat Q - \mB^\pi \hat Q)(s,a)$ as stated previously in Assumption \ref{ass:bound-q} gives the final claim.
\end{proof}

\begin{proof}[Proof of Theorem \ref{non-linear-theo}]
	similar to the result of non-linear version of CQL \cite{kumarConservativeQLearningOffline2020}. Our proof leverages the neural tangent kernel assumption \cite{jacotNeuralTangentKernel2020} that enables us to reduce the problem to a linear setup, where we demonstrated our result in theorem \ref{linear-theo}. first we express $\psi^{k+1}$ obtained by a one-step gradient update under equation \ref{eq:dr-combo-alt} with step size $\eta$:
	\begin{displaymath}
		\begin{aligned}
			\psi^{k+1} = & \psi^k - \eta \beta^k \left(\E_{\rho(s,a)}\left[\gQ(s,a)\right]-\E_{d(s,a)}\left[\gQ(s,a)\right]\right)                                        \\
			             & - \eta \E_{d_f(s,a)}\left[\left(\hat Q^k - \hat \mB \hat Q^k\right)\cdot\gQ(s,a)\right]                                                        \\
			             & - \eta \alpha^k \E_{d_f(s,a)}\left[\frac{(\gQ(s,a)-\E_{d_f(s,a)}[\gQ(s,a)])(\hat Q^k(s,a)-\E_{d_f(s,a)}[\hat Q^k(s,a)])}{\sqrt{n}\CTwo}\right] \\
			=            & \psi^k - \eta \beta^k \left(\E_{\rho(s,a)}\left[\gQ(s,a)\right]-\E_{d(s,a)}\left[\gQ(s,a)\right]\right)                                        \\
			             & - \eta \E_{d_f(s,a)}\left[\left(\hat Q^k - \mB \hat Q^k\right)\cdot\gQ(s,a)\right]                                                             \\
			             & - \eta \alpha^k n^{-\frac 1 2}\frac{\cov(\gQ,\hat Q^k)}{\CTwo}\E_{d_f(s)}
			\\
			             & + \eta \E_{d_f(s,a)}\left[\Delta(s,a)\cdot\gQ(s,a)\right]
		\end{aligned}
	\end{displaymath}
	where $\Delta(\cdot,\cdot)$ is given by $\Delta(s,a) :
		= f(\hat \mB^\pi_{\widetilde \mM} \hat Q - \mB^\pi \hat Q)(s,a)
		+ (1-f)(\hat \mB^k_{\widehat \mM} \hat Q - \mB^\pi \hat Q)(s,a)$.

	Under small learning rates, i.e., $g=\mO(1)$, as has been shown previously on works on neural tangent kernel towards explaining over-parameterized one-hidden-layer neural network (\cite{jacotNeuralTangentKernel2020}), we can taylor-expand the Q-update as follows:
	\begin{displaymath}
		\begin{aligned}
			 & \hat Q^{k+1}(s,a)                                                                                                           \\
			 & \approx \hat Q^k(s,a) + (\psi^{k+1}-\psi^k)^T \gQ(s,a)                                                                      \\
			 & = \hat Q^k(s,a)                                                                                                             \\
			 & -\eta \beta^k \left(\E_{\rho(s',a')}\left[\gQ(s',a')^T\gQ(s,a)\right]-\E_{d(s',a')}\left[\gQ(s',a')^T\gQ(s,a)\right]\right) \\
			 & - \eta \E_{d_f(s',a')}\left[\left(\hat Q^k - \mB \hat Q^k\right)\cdot\gQ(s',a')^T\gQ(s,a)\right]                            \\
			 & - \eta \alpha^k \sqrt{1/n} \frac{\cov(\gQ,\hat Q^k)}{\CTwo}\gQ(s,a)                                                         \\
			 & + \eta \E_{d_f(s',a')}\left[\Delta(s,a)\cdot\gQ(s',a')^T\gQ(s,a)\right]
		\end{aligned}
	\end{displaymath}
	For simplicity, denote $\mat K^k = \left(\gQ\right)^T\gQ$ as the NTK kernel matrix of the Q-function at iteration $k$. Then $Q^{k+1}$ is given by

	\begin{displaymath}
		\begin{aligned}
			\hat Q^{k+1} & = \hat Q^{k}-\eta \beta^k \mat K^k\left(\mat{U-V}\right)
			+\eta\mat {K^k D_f}\left(\mB^\pi \hat Q^k-\hat Q^k\right)-\eta\alpha^k\sqrt{1/n}\frac{\left(\mat{D_f - D_f^{\odot 2}}\right)\mat K^k \hat Q^k}{\CTwo} \\
			             & + \eta \mat {K^k D_f}\Delta.
		\end{aligned}
	\end{displaymath}
	\begin{equation}
		\begin{aligned}
			 & (\mu(s)\cdot\pi(a|s))^T\hat Q^{k+1}                                                                                                                                                                                    \\
			 & = \underbrace{(\mu(s)\cdot\pi(a|s))^T\hat Q^{k}+\eta(\mu(s)\cdot\pi(a|s))^T\mat {K^k D_f}\left(\mB^\pi \hat Q^k-\hat Q^k\right)}_{\text{(1) the unpenalized value}}                                                    \\
			 & -\underbrace{\eta\alpha^k(\mu(s)\cdot\pi(a|s))^T\sqrt{1/n}\frac{\left(\mat{D_f - D_f^{\odot 2}}\right)\mat K^k \hat Q^k}{\CTwo}-\eta \beta^k (\mu(s)\cdot\pi(a))^T\mat K^k\left(\mat{U-V}\right)}_{\text{(2) penalty}}
			\\
			 & + \underbrace{\eta(\mu(s)\cdot\pi(a|s))^T \mat {K^k D_f}\Delta}_{\text{(3) over-estimation}}.
		\end{aligned}
		\label{eq:non-lin-q}
	\end{equation}

	In order that the LHS is indeed a lower bound, we require that $-(2)+(3)<0$. It suffices to have $\alpha$ and $\beta$ satisfies
	\begin{equation}
		\begin{aligned}
			\beta & \ge \mat{(U-V)^{-1}}\mat D_f	\max_{s,a}\Delta(s,a) - \eta\alpha^k(\mu(s)\cdot\pi(a|s))^T\sqrt{1/n}\frac{\left(\mat{D_f - D_f^{\odot 2}}\right)\mat K^k \hat Q^k}{\CTwo}
		\end{aligned}
	\end{equation}

	similar to the proof of gap-expanding property (see Theorem \ref{thrm:gap-expanding}) above, we can bound $\Delta(\cdot,\cdot)$ with $(1-f)\left[|r-r_{\widehat \mM}|
		+\frac {2\gamma R_{\max}}{1-\gamma}\TV(T,\hat T)\right] + \frac{fC_{r,T,\delta}R_{\max}}{(1-\gamma)\sqrt{|\mD|}}$

\end{proof}

\end{document}